\documentclass[10pt,journal,compsoc]{IEEEtran}
\ifCLASSOPTIONcompsoc
  \usepackage[nocompress]{cite}
\else
  \usepackage{cite}
\fi

\ifCLASSINFOpdf
   \usepackage[pdftex]{graphicx}
\else
   \usepackage[dvips]{graphicx}
\fi

\usepackage{amsmath}

\usepackage{algorithmic}

\usepackage{array}

\ifCLASSOPTIONcompsoc
 \usepackage[caption=false,font=footnotesize,labelfont=sf,textfont=sf]{subfig}
\else
 \usepackage[caption=false,font=footnotesize]{subfig}
\fi

\usepackage{stfloats}

\ifCLASSOPTIONcaptionsoff
 \usepackage[nomarkers]{endfloat}
\let\MYoriglatexcaption\caption
\renewcommand{\caption}[2][\relax]{\MYoriglatexcaption[#2]{#2}}
\fi

\usepackage{url}

\usepackage{color}
\usepackage{array}
\usepackage{float}
\usepackage{booktabs}
\usepackage{multirow}
\usepackage{amsmath}
\usepackage{amssymb}
\usepackage{graphicx}
\usepackage{mathtools}
\usepackage[breaklinks=true,bookmarks=false]{hyperref}
\usepackage{hyperref}
\hypersetup{
    colorlinks=true,
    linkcolor=black,     
    urlcolor=black,
}

\usepackage{amsthm}
\makeatletter
\usepackage{xspace}
\DeclareRobustCommand\onedot{\futurelet\@let@token\@onedot}
\def\@onedot{\ifx\@let@token.\else.\null\fi\xspace}
\def\eg{\emph{e.g}\onedot} 
\def\ie{\emph{i.e}\onedot}

\def\etal{\emph{et al}\onedot}
\newcommand{\mypar}[1]{\smallskip\noindent {\bf #1}\enskip}

\newcommand{\revise}[1]{\textcolor{black}{#1}}
\newcommand{\minor}[1]{\textcolor{black}{#1}}
\newcommand{\bX}{\bm{X}}
\newcommand{\bY}{\bm{Y}}
\newcommand{\bx}{\bm{x}}
\newcommand{\by}{\bm{y}}
\newcommand{\bg}{\bm{\gamma}}

\newcommand{\bH}{\bm{H}}
\newcommand{\bU}{\bm{U}}
\newcommand{\bb}{\bm{\beta}}

\newcommand{\be}{\bm{\varepsilon}}
\newcommand{\ci}[1]{\tiny{$\pm #1$}}

\providecommand{\tabularnewline}{\\}
\floatstyle{ruled}
\newfloat{algorithm}{tbp}{loa}
\providecommand{\algorithmname}{Algorithm}
\floatname{algorithm}{\protect\algorithmname}
\usepackage{ragged2e}
\usepackage{epsfig}
\usepackage{graphicx}
\usepackage{caption}
\usepackage{multicol}
\usepackage{bm}
\theoremstyle{remark}
\newtheorem*{rem*}{\protect\remarkname}
\theoremstyle{plain}
\newtheorem{thm}{\protect\theoremname}
\theoremstyle{plain}

\theoremstyle{plain}
\newtheorem{lem}[thm]{\protect\lemmaname}
\theoremstyle{plain}

\theoremstyle{definition}

\theoremstyle{plain}
\newtheorem{prop}[thm]{\protect\propositionname}
\providecommand{\assumptionname}{Assumption}
\providecommand{\corollaryname}{Corollary}
\providecommand{\definitionname}{Definition}
\providecommand{\lemmaname}{Lemma}
\providecommand{\propositionname}{Proposition}
\providecommand{\remarkname}{Remark}
\providecommand{\theoremname}{Theorem}
\usepackage{soul}

\def\eg{\textit{e.g.}}
\def\ie{\textit{i.e.}}

\def\etal{\textit{et al.}}

\begin{document}

\title{How to Trust Unlabeled Data? Instance Credibility Inference for Few-Shot Learning}

\author{Yikai~Wang,~Li~Zhang,~Yuan~Yao,~and~Yanwei~Fu
\IEEEcompsocitemizethanks{
\IEEEcompsocthanksitem Yuan Yao and Yanwei Fu are  the co-corresponding authors.
\IEEEcompsocthanksitem Yikai Wang, Li Zhang and Yanwei Fu are with the School of Data Science, Fudan University, and Shanghai Key Lab of
Intelligent Information Processing, Fudan University. Yanwei Fu is also with the MOE Frontiers Center for Brain Science, Fudan University. E-mail: \{yikaiwang19, lizhangfd, yanweifu\}@fudan.edu.cn
\IEEEcompsocthanksitem Yuan Yao is with the Department of Mathematics, Hong Kong University of Science and Technology.
E-mail: yuany@ust.hk
}
}

\IEEEtitleabstractindextext{
\begin{abstract}
\justifying
Deep learning based models have excelled in many computer vision tasks and appear to surpass humans' performance.
However, these models require an avalanche of expensive human labeled training data and many iterations to train their large number of parameters.
This severely limits their scalability to the real-world long-tail distributed categories, some of which are with a large number of instances, but with only a few manually annotated.
Learning from such extremely limited labeled examples is known as Few-Shot Learning (FSL).
Different to prior arts that leverage meta-learning or data augmentation strategies to alleviate this extremely data-scarce problem, this paper presents a statistical approach, dubbed Instance Credibility Inference (ICI) to exploit the support of unlabeled instances for few-shot visual recognition.
Typically, we repurpose the self-taught learning paradigm to predict pseudo-labels of unlabeled instances \revise{with} an initial classifier trained from the \emph{few shot} and then select the most confident ones to augment the training set to re-train the classifier.
\revise{This is achieved by constructing} a (Generalized) Linear Model (LM/GLM) with incidental parameters to model the mapping from (un-)labeled features to their (pseudo-)labels, in which the sparsity of the incidental parameters indicates the credibility of the corresponding pseudo-labeled instance.
We rank the credibility of 
\revise{pseudo-labeled instances}
along the regularization path of their corresponding incidental parameters, and the most trustworthy  pseudo-labeled examples are preserved as the augmented labeled instances.
This process is repeated until all the unlabeled samples are
included in the expanded training set.
Theoretically, under the conditions of \emph{ restricted eigenvalue, irrepresentability, and large error},
our approach is guaranteed to collect \emph{all} the correctly-predicted pseudo-labeled instances from the noisy pseudo-labeled set.
\revise{
Extensive experiments under two few-shot settings show the effectiveness of our approach on four widely used few-shot visual recognition benchmark datasets including \textit{mini}ImageNet, \textit{tiered}ImageNet, CIFAR-FS, and CUB.
}
Code and models are released at \url{https://github.com/Yikai-Wang/ICI-FSL}.
\end{abstract}

}

\maketitle

\IEEEdisplaynontitleabstractindextext

\IEEEpeerreviewmaketitle

\IEEEraisesectionheading{\section{Introduction}\label{sec:introduction}}
\begin{figure*}[hbt!]
\begin{centering}
\includegraphics[width=2\columnwidth]{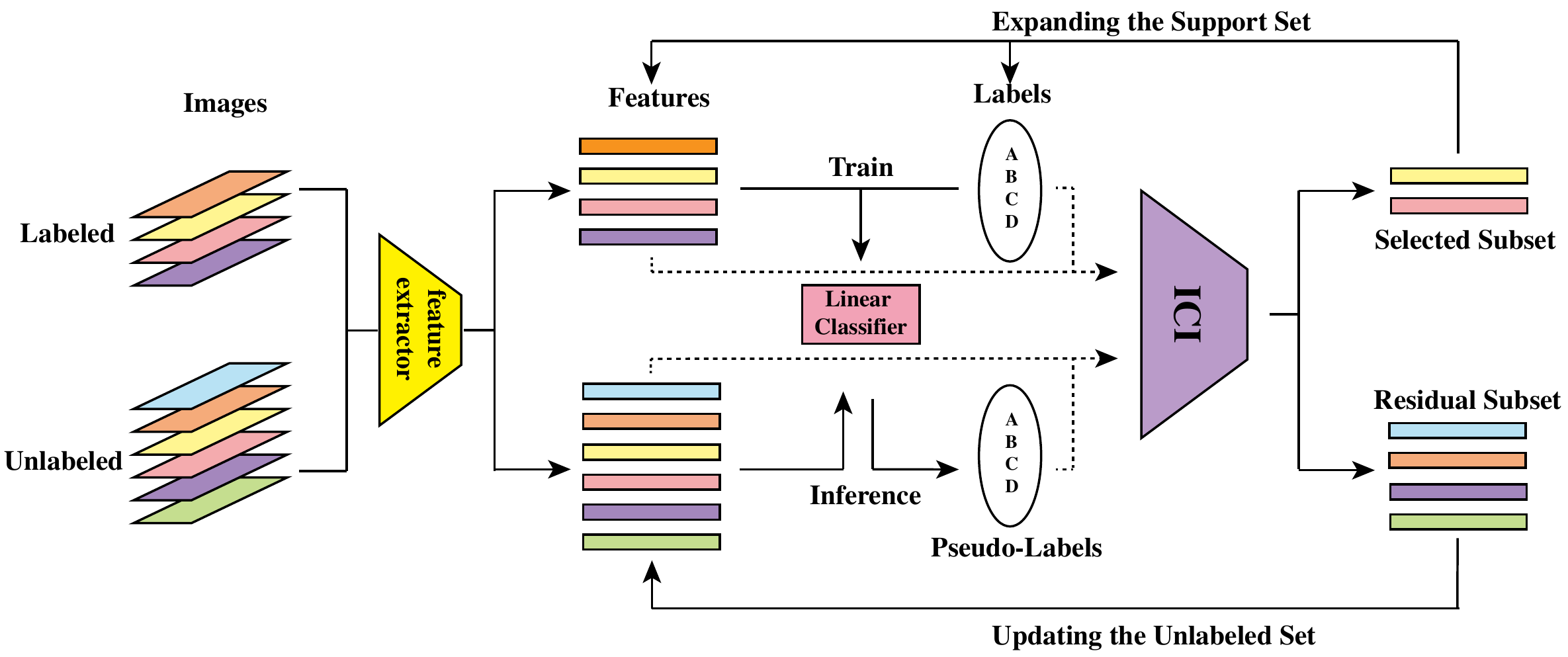}
\par\end{centering}
\caption{\label{fig:framework} 
The inference process of our proposed framework.
We extract features of each labeled and unlabeled instance, 
train a linear classifier with the support set, 
provide pseudo-label for the unlabeled instances, 
and use ICI to select the most trustworthy subset to expand the support set.
This process is repeated until all the unlabeled data are included in the support set.}
\end{figure*}

\IEEEPARstart{H}{umans} \revise{are able to efficiently perform visual recognition by learning from a single example or a single exposure.}
For example, children have no problem of forming the concept of ``giraffe'' by only taking a glance from a picture in a book~\cite{wang2020generalizing}, 
or hearing its description as looking like a deer with a long neck~\cite{zhang2017learning}.
In contrast, 
\revise{the most successful recognition systems, deep learning based in particular}~\cite{krizhevsky2012imagenet,simonyan2014very,he2016deep,huang2017densely} still highly rely on an avalanche of labeled training data.
\revise{This is problematic. It inevitably} increases the burden in rare data collection (\eg~accident data in the autonomous driving scenario) and
expensive data annotation (\eg~disease data for medical diagnose), 
and more fundamentally limits their scalability to open-ended learning of the long tail categories in the real-world.

Motivated by these observations, there has been a recent resurgence of research interest in few-shot learning~\cite{finn2017model,snell2017prototypical,sung2018learning,vinyals2016matching}.
It aims to recognize new objects with extremely limited training data for each category. 
\revise{
To address this issue, the key idea is to train the model by transferring the knowledge from a disjoint but relevant dataset.
Typically, the model trained on the \emph{source/base} dataset, which includes many labeled instances, is  expected  to be well generalizable to  the \emph{ target/novel} dataset with only scarce labeled data. 
}

A key challenge for few-shot learning is how to transfer the learned knowledge to new tasks.
The simplest strategy 
is fine-tuning~\cite{yosinski2014transferable}, utilizing the limited training instances to update the learned models.
Practically, it inevitably causes severely overfitting as one or a few instances are insufficient to model the data distributions of the novel classes.
Data augmentation and regularization techniques~\cite{chen2019image,chen2019multi} can alleviate overfitting in such a limited-data regime, but they do not solve it.
Several recent efforts are made in leveraging learning to learn, or meta-learning~\cite{lemke2015metalearning} paradigm by simulating the few-shot scenario in the training process~\cite{vinyals2016matching, snell2017prototypical, oreshkin2018tadam, sung2018learning, sung2017learning, finn2017model, li2017meta, nichol2018first, rusu2018meta}. 
However, Chen \etal~\cite{DBLP:journals/corr/abs-1904-04232} empirically argues that such a learning paradigm often results in inferior performance compared to a simple baseline with a linear classifier coupled with a deep feature extractor.
This phenomenon is also verified in~\cite{Liu_2020_CVPR_Workshops}.

In real-world applications,
unlabeled instances are easier and cheaper to obtain, comparing to the labeled instances which usually require expensive human annotation.
Potentially we could utilize the unlabeled instances to alleviate the data-scarce problem and help learn the few-shot model.
Specifically,
two types of strategies resort to model the data distribution of novel category beyond traditional \emph{inductive} few-shot learning:
(i) semi-supervised few-shot learning (SSFSL)~\cite{liu2018learning,ren2018meta,sun2019learning} supposes that we can utilize 
unlabeled data to help to learn the 
model;
furthermore,
(ii) transductive inference~\cite{joachims1999transductive} for few-shot learning (TFSL)~\cite{liu2018learning,qiao2019transductive} assumes we can access all the test data, rather than evaluate them one by one in the inference process. 
In other words, the few-shot learning model can utilize the data distributions of testing examples.

Self-taught learning~\cite{self-taught-learning} is one of the most straightforward ways to leverage the information of unlabeled data. 
Typically, a trained classifier infers the pseudo labels of unlabeled data, which are further taken to update the classifier. 
\revise{
Nevertheless, the inferred pseudo-labels may be very noisy;
the wrongly labeled instances may jeopardize the performance of the classifier. 
}
It is thus essential to investigate the labeling confidence of each unlabeled instance.

To this end, we present a statistical approach, dubbed Instance Credibility Inference (ICI) to exploit the distribution support of unlabeled instances for few-shot learning.
Specifically, we first train a 
simple linear classifier (\eg~logistic regression, or linear support vector machine)
with the labeled few-shot examples and use it to infer the pseudo-labels for the unlabeled instances.
\revise{
The credibility of each pseudo-labeled instances is measured by the proposed ICI.
Then a most trustworthy subset can be selected and expanded into the support set.
}
The simple classifier thus can be progressively updated (re-trained) by the expanded support set and further infer pseudo-labels for the unlabeled data.
This process is repeated until all the unlabeled instances are iteratively selected to expand the support set, \ie~the pseudo-label of each unlabeled instance is converged.
The schematic illustration is shown in Fig.~\ref{fig:framework}.

Basically, we re-purpose the standard self-taught learning algorithm by our proposed ICI algorithm.
How to select the pseudo-labeled data and exclude the wrongly-predicted samples,~\ie, excluding the noise introduced by the self-taught learning strategy?
Our intuition is that the credibility criteria can neither solely rely on the manifold structure of the feature space (\eg~instances that are close to labeled instances under a certain distance metric) nor the 
label space (\eg~prediction score provided by the classifier).
Instead, 
we propose to solve the hypothesis of (generalized) linear models (\ie~linear regression or logistic regression) {by progressively increasing} the sparsity of the data-dependent incidental parameter~\cite{fan2018partial} until it vanishes.
Thus we can credit each pseudo-labeled instance by the sparsity of the corresponding incidental parameter.
We prove that under the conditions of 
restricted eigenvalue, irrepresentability, and large error,
our proposed method is able to collect \emph{all} the correctly-predicted pseudo-labeled instances.
We conduct extensive experiments on major few-shot learning benchmark datasets to validate the effectiveness of our proposed algorithm.

\noindent \textbf{Contributions.} 
The contributions of this work are as follows.

\noindent
(i) We present a statistical approach, dubbed Instance Credibility Inference (ICI) to exploit the distribution support of unlabeled instances for few-shot learning. 
Specifically, our model iteratively selects the pseudo-labeled instances according to its credibility measured by the proposed ICI for classifier training.

\noindent
(ii) We re-purpose the standard self-taught learning algorithm~\cite{self-taught-learning} by our proposed ICI. 
To measure the credibility of each pseudo-labeled instance, we solve the LM/GLM hypothesis by increasing the sparsity of the incidental parameter~\cite{fan2018partial} and 
regard the sparsity level as the credibility for each pseudo-labeled instance. 

\noindent
(iii) Under the conditions of 
restricted eigenvalue, irrepresentability, and large error,
we can prove that our method collects \emph{all} the correctly-predicted pseudo-labeled instances.

\noindent
(iv) Extensive experiments under two few-shot settings show \revise{the effectiveness of our approach} on four widely used few-shot learning benchmark datasets including \textit{mini}ImageNet, \textit{tiered}ImageNet, CIFAR-FS, and CUB.

\noindent \textbf{Extensions.} 
A preliminary version of this work was published in~\cite{wang2020instance}. 
We have extended our conference version as follows.

\noindent
(i) 
We provide the theoretical analysis 
of ICI to answer the question that 
\textit{under what conditions can ICI find \textbf{all} the correctly-predicted instances}?

\noindent
(ii)
We show that our ICI can be extended to generalized linear models, in particular, a \emph{logistic regression model with sparse incidental parameters}. Particularly we show in our experiments the effectiveness of such a logistic regression model with sparsity regularization for ICI.

\section{Related work}
\subsection{Semi-supervised learning} 
Semi-supervised learning (SSL) aims to improve the learning performance with both labeled and unlabeled instances.
Basic assumptions in semi-supervised learning include continuity, cluster, and manifold assumptions.
Conventional approaches focus on finding decision boundaries with both labeled and unlabeled data~\cite{vapnik1998statistical,bennett1999semi,joachims1999transductive}, and avoiding to learn the ``wrong'' knowledge from the unlabeled data~\cite{li2014towards} based on specific hypothesis.
Recently, semi-supervised learning with deep learning models use consistency regularization~\cite{conf/iclr/LaineA17}, moving average technique~\cite{tarvainen2017mean} and adversarial perturbation regularization~\cite{miayto2016virtual} to train the model with a large amount of unlabeled data.
\revise{
The task of semi-supervised few-shot learning is an extension of addressing SSL in the setting of few-shot learning, where only limited labeled target instances are available. Critically,  as explained in~\cite{ren2018meta}, the vanilla SSL is solved in the standard supervised learning setting, whilst the SSFSL targets at addressing a transfer learning task.
}

\subsection{Self-taught learning}
Self-taught learning~\cite{self-taught-learning}, also known as self-training~\cite{NoisyStudent},
is a traditional semi-supervised strategy of utilizing unlabeled data to improve the performance of classifiers~\cite{amini2002semi,grandvalet2005semi}. 
Self-taught learning algorithms often start by training an initial recognition model and infer the pseudo-labels of unlabeled instances, then the pseudo-labeled instances are taken to re-train the recognition model with specific strategies~\cite{lee2013pseudo}.
Deep learning based self-taught learning strategy includes 
(i) directly training the neural network with both labeled instances and pseudo-labeled instances ~\cite{lee2013pseudo}, 
(ii) utilizing mix-up images between labeled instances and pseudo-labeled instances to synthesis training instances with less noise~\cite{arazo2019pseudo},
(iii) utilizing indirect ways to infer the pseudo-label of unlabeled instances (for example use label propagation constructed on the nearest-neighbor graph and select the trustworthy subset based on the entropy~\cite{iscen2019label}),
and 
(iv) methods that introducing inductive bias (\eg~adding a cluster assumption on the feature space and re-weight the pseudo-labeled instances based on this assumption~\cite{shi2018transductive})
One of the key points in self-taught learning algorithms is how to reduce the noise introduced by the imperfect recognition models.
Different from previous works, 
we measure the credibility of each pseudo-labeled instance by a statistical algorithm.
Only the most trustworthy subset is employed to re-train the recognition model jointly with the labeled instances.

\subsection{Learning with noisy labels}
\revise{
There are many works on learning with noisy labels~\cite{angluin1988learning}.  
The noisy labels indicate that the provided label may not be the true class of the instance. 
Such noise may come from the annotation errors, mismatching of the search engine, or the pseudo-label in the self-taught learning process.
Typical approaches in learning with noisy labels~\cite{song2020learning} include robust loss function~\cite{ghosh2017robust}, robust architecture~\cite{goldberger2016training}, robust regularization~\cite{jenni2018deep}, loss adjustment~\cite{chang2017active, ICML2019_UnsupervisedLabelNoise}, and sample selection~\cite{song2020robust}. }

\revise{Sample selection aims to find clean subset from the noisy dataset to prevent the negative impact of noise.
In deep learning based approaches, a popular assumption is that when the network is under-fitted, the loss of noisy samples are larger than clean samples. 
O2u-net~\cite{huang2019o2u} cyclically changes the learning rate of the network to satisfy the under-fitting condition,   measure the loss of each sample and exclude the noisy subset.
ODD~\cite{song2020robust} uses large learning rate to exclude the samples with higher losses.}

\revise{However, almost all of these algorithms are based on the inherent assumption that a large number of training samples are accessible.
Further, they mainly focus on the standard supervised learning setting.
In contrast, SSFSL focuses on the transfer learning tasks.
}

\subsection{Few-shot learning}
Few-shot learning aims to recognize novel visual categories from very few labeled examples.
Recent efforts mainly follow the meta-learning strategy.
That is, by simulating the few-shot scenario in the training process, algorithms are learning to learn with limited data.
\revise{
We can roughly categorize existing works on few-shot learning into the following groups.
(i)
Learning robust and discriminative distance metrics,
including weighted nearest neighbor classifier (\eg~Matching Network~\cite{vinyals2016matching}), 
finding robust prototype for each class (\eg~Prototypical Network~\cite{snell2017prototypical}),
learning task-dependent metrics (\eg~TADAM~\cite{oreshkin2018tadam}),
and learning parameterized metrics via neural networks~\cite{sung2018learning}.
(ii)
Finding the optimal initialization parameters that could rapidly adapt to specific task, including Meta-Critic~\cite{sung2017learning}, MAML~\cite{finn2017model}, Meta-SGD~\cite{li2017meta}, Reptile~\cite{nichol2018first}, and LEO~\cite{rusu2018meta}.
(iii) Data augmentation strategies aim to alleviate the problem of limited data by directly synthesising new data in the image level~\cite{chen2019image} or the feature level~\cite{chen2019multi}. 
Additionally, SNAIL~\cite{mishra2018a} utilizes the sequence modeling to create a new framework.
The proposed statistical algorithm is orthogonal 
and
potentially beneficial to these algorithms -- it is always worth increasing the training set by utilizing the unlabeled data with confidently predicted labels.
}

\subsection{Few-shot learning with unlabeled data}
\revise{
Recent works~\cite{hou2019cross,hu2020exploiting,lichtenstein2020tafssl,yang2020dpgn,hu2020leveraging,kye2020transductive} start to tackle few-shot learning with additional unlabeled instances.
Compared with the traditional inductive setting, algorithms trained with unlabeled instances have the chance to handle a more trustworthy empirical distribution.
Ren~\emph{et al.}~\cite{ren2018meta} utilized the unlabeled data to refine the prototype of each class.
Liu~\emph{et al.}~\cite{liu2018learning} utilized label propagation strategy to transfer labels based on the relative distances within labeled data and unlabeled data.
DPGN~\cite{yang2020dpgn} adopts contrastive comparisons to produce distribution representation.} 

\revise{Self-taught learning is also utilized in SSFSL.
For example, LST~\cite{sun2019learning} uses the self-taught learning strategy in the transductive inference setting and trains the model in a meta-learning manner.
CAN~\cite{hou2019cross} uses the self-taught learning to train the model repeatedly within the specific designed network.
TAFSSL~\cite{lichtenstein2020tafssl} reduces the dimension of sample features to get a simpler manifold and construct specific self-taught learning algorithm based on the low-dimensional manifold.
Compared with those algorithms, our approach is much simpler and theoretically guaranteed.
Unlike previous meta-learning algorithms which usually has pre-training, meta-training, and meta-test process~\cite{sun2019learning},
our approach only modifies the inference process.
}

\subsection{Incidental parameters}
Incidental parameters problem~\cite{neyman1948consistent} was tackled by the penalized estimation algorithms~\cite{fan2010selective}.
It assumes the existence of sparse data-dependent parameters in the estimation models.
For example, the linear regression model with incidental parameters follows 
$y_i=x_i^{\top}\beta^{*}+\gamma_i^{*}+\varepsilon_i$,
where $\left(x_i,y_i\right)$ denotes data input, $\beta^{*}$ is the traditional coefficients, $\varepsilon_i$ denotes the random noise and $\gamma_i^{*}$ is the introduced data-dependent incidental parameters.
Prior arts solve this problem by estimating the coefficients which are robust against the incidental parameters~\cite{neyman1948consistent, kiefer1956consistency, basu2011elimination, moreira2008maximum, fan2018partial}.
Fu~\emph{et al.}~\cite{fu2015robust} introduce the incidental parameter in robust ranking task.
In this paper, we propose to solve the few-shot learning problem based on the intuition that the incidental parameters indicate the credibility of pseudo-labeled instances.
We do so by utilizing a weak estimation of coefficients to enlarge the influence of incidental parameters and transfer a 
\revise{``\emph{generalized linear model with incidental parameters}'' into a normal ``\emph{generalized linear}'' model }
whose coefficients are the former incidental parameters.
Then we estimate the incidental parameters along the regularization path to get the credibility of the corresponding instance.
We further provide the theoretical properties of ICI.
\section{Methodology}
\subsection{Problem formulation}
Here we define the few-shot learning problem mathematically.
We are provided a base category set and a novel category set, denoted as $\mathcal{C}_{base}$ and $\mathcal{C}_{novel}$, respectively.
The two category sets have no common category\footnote{Note that here and below we ignore another validation set for model selection since we could regard it as the novel set that is accessible in the training process.}, \ie, $\mathcal{C}_{base}\bigcap\mathcal{C}_{novel}=\emptyset$. 
Within each category set, we have a corresponding dataset, denoted as $\mathcal{D}_{base}=\left\{ \left(\bm{I}_{i},y_{i}\right),y_{i}\in\mathcal{C}_{base}\right\}$ and $\mathcal{D}_{novel}=\left\{ \left(\bm{I}_{i},y_{i}\right),y_{i}\in\mathcal{C}_{novel}\right\}$, respectively.
With the above notations, few-shot learning algorithms aim to train on $\mathcal{D}_{base}$ and contain the capacity of rapidly 
\revise{adapting}
to $\mathcal{D}_{novel}$ with access to only one or a few labeled instances per class.

For evaluation, we adopt the standard \emph{$c$-way-$m$-shot} classification as \cite{vinyals2016matching} on $\mathcal{D}_{novel}$. 
Specifically, in each episode, we randomly sample $c$ classes to construct our category pool $\mathcal{C}$, that is $\mathcal{C}\sim \mathcal{C}_{novel},\left|\mathcal{C}\right|=c$;
and $s$ and $q$ labeled images per class are randomly sampled in $\mathcal{C}$ to construct the support set $\mathcal{S}$ and the query set $\mathcal{Q}$, respectively.
Thus we have $\left|\mathcal{S}\right|=c\times s$ and $\left|\mathcal{Q}\right|=c\times q$.
The classification accuracy is averaged on query sets $\mathcal{Q}$ of many
meta-testing episodes. In addition, we have unlabeled
data of novel categories $\mathcal{U}_{novel}=\left\{ \bm{I}_{u}\right\} $. 

\subsection{Self-taught learning from unlabeled data}
We recap the self-taught learning formalism~\cite{self-taught-learning} to tackle few-shot learning problem with unlabeled data.
Particularly,
denote $f\left(\cdot\right)$ as the feature extractor trained
on $\mathcal{D}_{base}$. 
In one episode,
one can train a supervised classifier
$g\left(\cdot\right)$ on the support set $\mathcal{S}$, 
and pseudo-labeling unlabeled data, $\hat{y}_{i}=g\left(f\left(\bm{I}_{u}\right)\right)$
with corresponding confidence $p_{i}$.
The most confident unlabeled instances will be further taken as additional data of corresponding classes in the support set $\mathcal{S}$. 
Thus we obtain the updated supervised classifier $g\left(\cdot\right)$. 
To this end, few-shot classifier acquires additional training instances, 
and thus its performance can be improved. 

However, it is problematic if directly utilizing self-taught learning in few-shot cases. 
Particularly, the supervised classifier $g\left(\cdot\right)$ is only trained by a few instances. 
The unlabeled instances with high confidence may not be correctly categorized, and the classifier will be updated by some wrong instances. 
Even worse, one can not assume the unlabeled instances follows the same class labels or generative distribution as the labeled data.
Noisy instances or outliers may also be utilized to update the classifiers. 
To this end, we propose a systematical algorithm: 
Instance Credibility Inference (ICI) to reduce the noise.

\subsection{Instance credibility inference (ICI)}
To measure the credibility of predicted labels over unlabeled data, we introduce a hypothesis of linear model by regressing each instance from feature to label spaces.  
Particularly, given $n$  instances of $c$ classes,
$\mathcal{S}=\left\{ \left(\bm{I}_{i},y_{i},\bx_{i}\right),y_{i}\in\mathcal{C}_{novel}\right\} $, where $y_i$ is the ground truth when $\bm{I}_{i}$ comes from the support set, or the pseudo-label when $\bm{I}_{i}$ comes from the unlabeled set; \revise{$\bm{x}_{i}$ is the feature vector of instance $i$}. We employ a simple linear regression model to ``predict'' the class label,
\begin{equation}
\by_{i}=\bm{x}_{i}^{\top}\bb^{*}+\bg_{i}^{*}+\bm{\varepsilon}_{i}\label{eq:lm},
\end{equation}
where $\bm{\beta}^{*}\in\mathcal{\mathbb{R}}^{d\times c}$ is the coefficient matrix;
$\bm{x}_{i}\in\mathcal{\mathbb{R}}^{d\times1}$; 
$\bm{y}_{i}$ is $c$ dimension one-hot vector denoting the class label of instance $i$,
and \minor{
$\varepsilon_{ij}$ is independent sub-Gaussian noise of zero mean and variance bounded by $\sigma^2$
}. 
Note that to facilitate the computations, 
we employ Locally Linear Embedding (LLE)~\cite{roweis2000nonlinear} to reduce the dimension of extracted feature $f(\bm{I}_{i})$ to $d$. 

Inspired by incidental parameters~\cite{fan2018partial}, 
we introduce $\gamma_{i,j}^{*}$ to amend the chance of instance $i$ belonging to class  $j$. The
larger \revise{magnitude of} $\left\Vert\gamma_{i,j}^{*}\right\Vert$, the higher difficulty in attributing instance $i$ to  class  $j$.

Consider the linear regression model for all instances, we are solving the problem of
\begin{equation}
\underset{\bm{\beta},\bm{\gamma}}{\mathrm{argmin}}\sum_{i=1}^{n}\left[\frac{1}{2}\left\Vert\bm{y}_{i}-\bm{x}_{i}^{\top}\bm{\beta}-\bm{\gamma}_{i}\right\Vert_{2}^{2}+\lambda R\left(\bm{\gamma}_{i}\right)\right]\label{eq:loss_func_instances},
\end{equation}
where $R\left(\cdot\right)$ is the sparsity penalty, \eg, $R\left(\bg_i\right)=\sum_{j=1}^c\left|\bg_{i,j}\right|$.
By re-writing Eq.~\eqref{eq:loss_func_instances} in a matrix form, we are thus solving the problem of
\begin{equation}
\left(\hat{\bm{\beta}},\hat{\bm{\gamma}}\right)=\underset{\bm{\beta},\bm{\gamma}}{\mathrm{argmin}}\frac{1}{2}\left\Vert\bm{Y}-\bm{X}\bm{\beta}-\bm{\gamma}\right\Vert _{\operatorname{F}}^{2}+\lambda R\left(\bm{\gamma}\right)\label{eq:loss_func},
\end{equation}
where $\left\Vert \cdot \right\Vert _{\operatorname{F}}^{2}$ denotes the Frobenius norm. $\bm{Y}=[\bm{y}_{i}^{\top}]^{\top}\in\mathcal{\mathbb{R}}^{n\times c}$ and
$\bm{X}=[\bm{x}_{i}]^{\top}\in\mathcal{\mathbb{R}}^{n\times d}$
indicate label and feature input respectively. $\bm{\gamma}=[\bm{\gamma}^{\top}_{i}]^{\top}\in\mathcal{\mathbb{R}}^{n\times c}$
is the incidental matrix.
$\lambda$ is the coefficient of the penalty term  $R\left(\cdot \right)$.
To solve Eq.~\eqref{eq:loss_func}, we find the derivative with respect to $\bm{\beta}$ and make it equal to $0$, then we have
\begin{equation}
\hat{\bb}=\left(\bm{X}^{\top}\bm{X}\right)^{\dagger}\bm{X}^{\top}\left(\bm{Y}-\bg\right)\label{eq:beta},
\end{equation}
\noindent where $\left(\cdot\right)^{\dagger}$ denotes the Moore-Penrose pseudo-inverse. 
Note that 
(i) we are interested in utilizing $\bg$
to measure the credibility of each instance along its regularization path, rather than estimating
$\hat{\bb}$, since the linear regression model is not good enough
for classification in general;
(ii) the $\hat{\bb}$ also relies
on the estimation of $\bg$. 
To this end, we take Eq.~\eqref{eq:beta}
into Eq.~\eqref{eq:loss_func} and solve the problem
as
\begin{equation}
\underset{\bg\in\mathbb{R}^{n\times c}}{\mathrm{argmin}}\frac{1}{2}\left\Vert \bm{Y}-\bm{H}\left(\bm{Y}-\bg\right)-\bg\right\Vert _{\operatorname{F}}^{2}+\lambda R\left(\bg\right),
\end{equation}
where 
$\bH=\bm{X}\left(\bm{X}^{\top}\bm{X}\right)^{\dagger}\bm{X}^{\top}$.
We further define $\tilde{\bm{X}}=\bm{I}-\bm{H}$ and $\tilde{\bm{Y}}=\tilde{\bm{X}}\bm{Y}$.
Then the above equation can be simplified as
\begin{equation}
\hat{\bm{\gamma}} = \underset{\bg\in\mathbb{R}^{n\times c}}{\mathrm{argmin}}\frac{1}{2}\left\Vert \tilde{\bm{Y}}-\tilde{\bm{X}}\bg\right\Vert _{\operatorname{F}}^{2}+\lambda R\left(\bg\right),\label{eq:penalty}
\end{equation}
which is a multi-response regression problem. 

Particularly, we regard $\hat{\bm{\gamma}}$ as a function of $\lambda$. 
When $\lambda$ changes from $0$ to $\infty$, the sparsity of $\hat{\bm{\gamma}}$ is increased until all of its elements are forced to vanish.
Further, we use the penalty $R\left(\bg\right)$ to encourage $\bg$ vanishes row by row, \ie, instance by instance. 
For example, $R\left(\bg\right)=\sum_{i=1}^n\sum_{j=1}^c\left|\bg_{i,j}\right|$ or $R\left(\bg\right)=\sum_{i=1}^n\left\Vert\bg_{i}\right\Vert_2$.
Moreover, the penalty \revise{tends to} vanish the subset of $\tilde{X}$ with the lowest deviations, indicating less discrepancy between the prediction and the ground truth.
Hence we could rank the pseudo-labeled data by the \emph{smallest} $\lambda$ value when the corresponding $\hat{\gamma}_i$ vanishes. 
As shown in one toy example of Figure~\ref{fig:illu}, the $\hat{\bm{\gamma}}$ value of the instance denoted by the red line vanishes first, and thus it is the most trustworthy sample by our algorithm. 

We seek the best subset by checking the regularization path, \ie~$\hat{\bm{\gamma}}(\lambda)$ as $\lambda$ varies, 
which can be easily configured by 
a block coordinate descent algorithm
implemented in Glmnet~\cite{simon2013blockwise}. 
Specifically,
we can find $\lambda_{max}=\underset{i}{\max}\left\Vert\tilde{\bm{X}}_{\cdot i}^{\top}\tilde{\bm{Y}}\right\Vert _{2}/n$
to guarantee \revise{that the solution of Eq.~\eqref{eq:penalty} all equals to 0}.
Then we can get a list of $\lambda$s from $0$ to $\lambda_{max}$. 
We solve a specific Eq.~\eqref{eq:penalty} with each $\lambda$,
and get the regularization path of $\bg$ along the way.

\begin{figure}
\begin{centering}
\includegraphics[width=0.8\columnwidth]{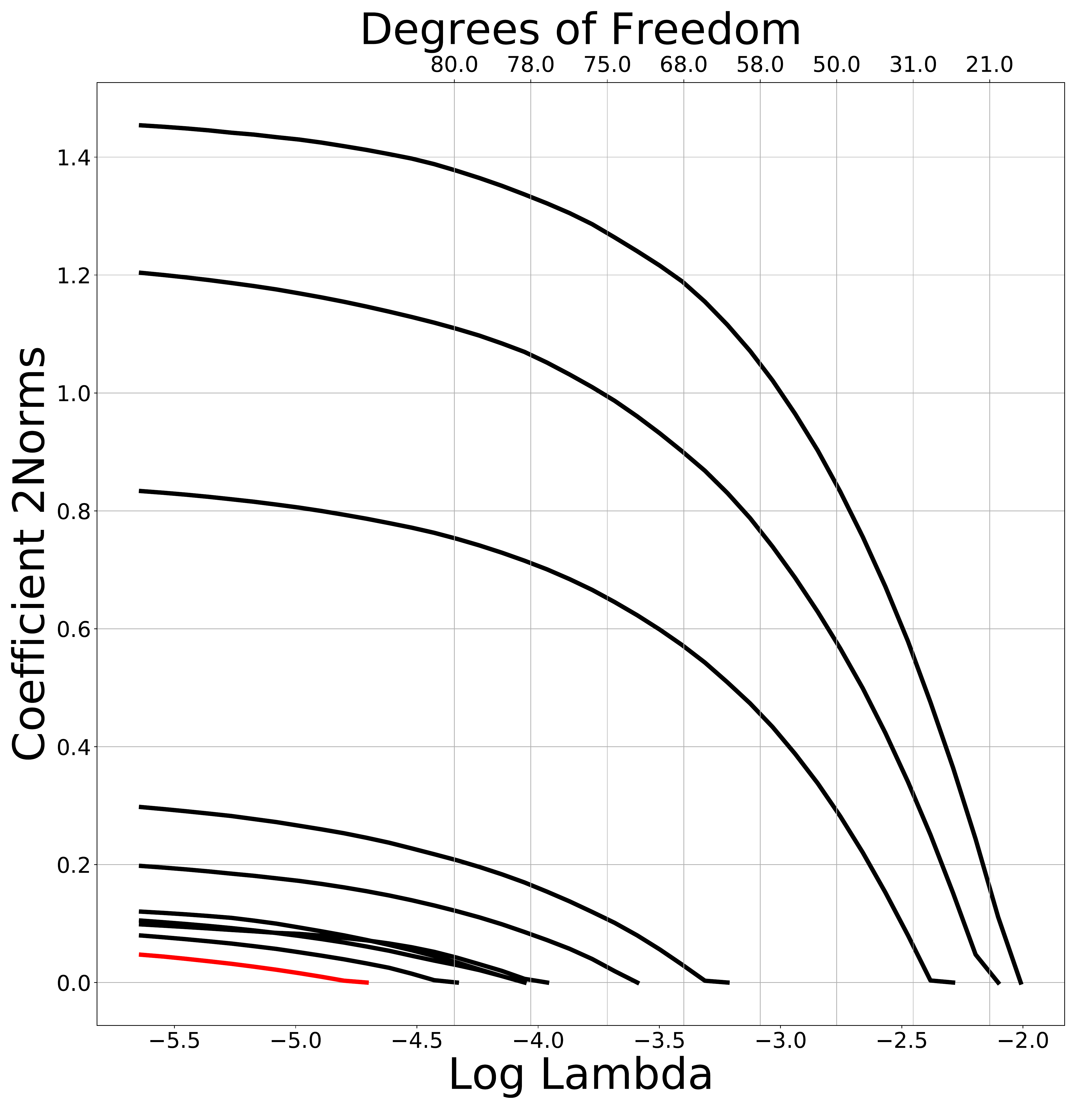}
\par\end{centering}
\caption{\label{fig:illu}
Regularization path of $\lambda$ on ten samples. Red line is corresponding to the most trustworthy sample suggested by our ICI algorithm.}
\end{figure}

\subsection{Extension to logistic regression\label{sec:extension-lr}}
\begin{figure*}[!ht]
\includegraphics[width=1\textwidth]{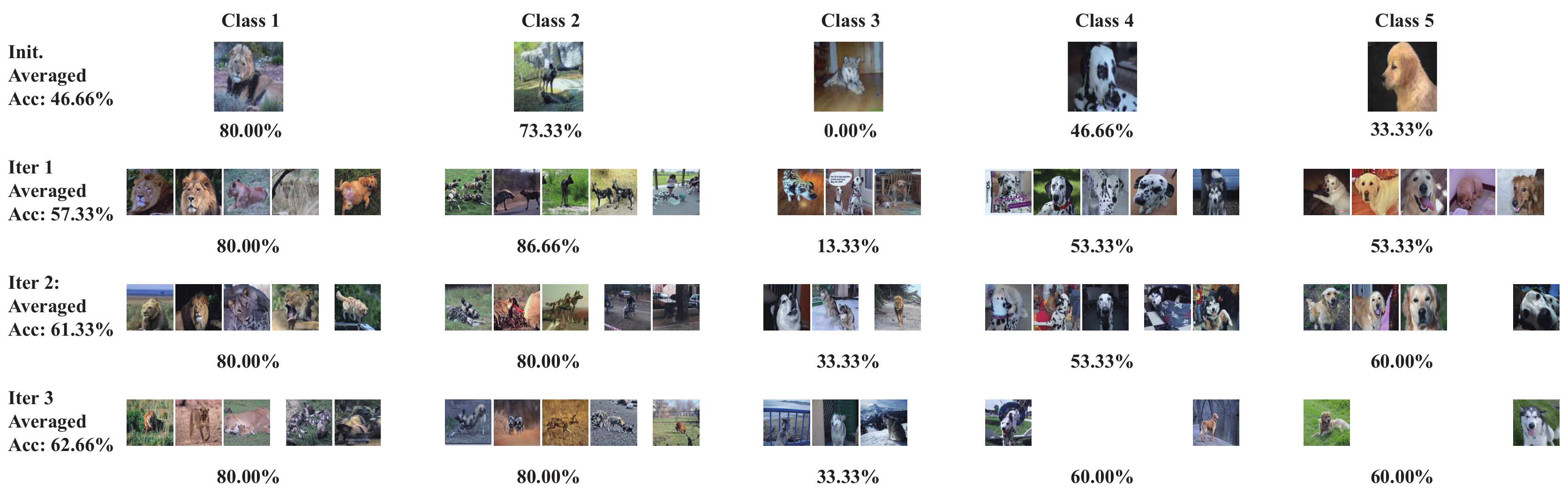} 
\caption{\label{fig:qualitative-images}
\revise{New images selected per class in each iteration of an inference episode on \textit{mini}ImageNet.
The averaged test accuracy is on the left, while the test accuracy of each class is listed at the bottom of the corresponding images in each iteration.
In each iteration, the correctly-predicted instances of each class are placed on the left, and vice versa on the right.
For each class, we select $5$ images at most.
Note that in some iteration the number of the left unlabeled instances of classes is smaller than $5$.
The remaining images are incorrectly predicted in the other classes.
}
}
\end{figure*}

In the above section, we develop ICI with a linear regression model.
But the basic idea of measuring credibility of pseudo-labeled instance as the sparsity level of the corresponding incidental parameters along the regularization path is general and not limited in the linear regression model.
To show this, in this section we extend ICI with generalized linear models, particularly, the logistic regression model.

Recall that we have 
$\bm{Y}=[\bm{y}_{i}^{\top}]^{\top}\in\mathcal{\mathbb{R}}^{n\times c}$ 
and
$\bm{X}=[\bm{x}_{i}]^{\top}\in\mathcal{\mathbb{R}}^{n\times d}$
as our label matrix and feature matrix,
respectively.
We use 
$\bm{\beta}^{*}\in\mathcal{\mathbb{R}}^{d\times c}$ as the coefficient matrix and
$\bm{\gamma}^{*}=\left[\bm{\gamma}_{i}\right]\in\mathcal{\mathbb{R}}^{n\times c}$
as the incidental matrix.
Then our logistic model with incidental parameters can be formed as
\begin{equation}
\bY_{i,c} = \frac{\exp \left(\bm{X}_{i\cdot}\bm{\beta}^{*}_{\cdot c}+\bm{\gamma}^{*}_{i,c}\right)}{\sum_{l=1}^C\exp \left(\bm{X}_{i\cdot}\bm{\beta}^{*}_{\cdot l}+\bm{\gamma}^{*}_{i,l}\right)}+\be_{i,c}.
\label{eq:logit-origin}
\end{equation}
This could be reformulated into a standard logistic regression model with sparsity regularization.
Specifically, we define 
$\bar{\bm{X}}=\left(\bm{X}, \bm{I}\right)\in\mathbb{R}^{n\times(d + n)}$ and $\bar{\bm{\beta}}^{*}=\left(\bm{\beta}^{*}, \bm{\gamma}^{*}\right)^{\top}\in\mathbb{R}^{(d + n)\times c}$, in which $\bm{I}$ is the identity matrix.
Then we have
\begin{equation}
\bar{\bm{X}}_{i\cdot}\bar{\bm{\beta}}^{*}_{\cdot c}=
\left(\bm{X}_{i\cdot}, \bm{I}_{i\cdot}\right)
\left(\bm{\beta}^{*}_{\cdot c}, \bm{\gamma}^{*}_{\cdot c}\right)^{\top}=
\bm{X}_{i\cdot}\bm{\beta}^{*}_{\cdot c}+\bm{\gamma}^{*}_{i,c}.
\end{equation}
Hence we could reformulate Eq.~\eqref{eq:logit-origin} as
\begin{equation}
\bY_{i,c} = \frac{\exp \left(\bar{\bm{X}}_{i\cdot}\bar{\bm{\beta}}^{*}_{\cdot c}\right)}
{\sum_{l=1}^C\exp \left(\bar{\bm{X}}_{i\cdot}\bar{\bm{\beta}}^{*}_{\cdot l}\right)}+\be_{i,c},
\label{eq:logit-reformualted}
\end{equation}
which is exactly a logistic regression model. Our objective is the penalized negative log-likelihood function:
\begin{equation}
\label{eq:logistic-regression}
\begin{aligned}
\underset{\bar{\bm{\beta}}=\left(\bm{\beta}, \bm{\gamma}\right)^{\top}}{\mathrm{argmin}} &-
\frac{1}{n}\sum_{i=1}^n \left(\sum_{l=1}^{c}\bY_{i,l}\left(\bar{\bm{X}}_{i,\cdot}\bar{\bb}_{\cdot,l}\right)
-\log \left(\sum_{l=1}^{c}e^{ \bar{\bm{X}}_{i,\cdot}\bar{\bb}_{\cdot,l}}\right)\right)
\\
&+\lambda_1 R\left(\bb\right) + \lambda_2 R\left(\bg\right).
\end{aligned}    
\end{equation}
The algorithm for solving Eq.~\eqref{eq:logistic-regression} is well established~\cite{zhu1997algorithm,fan2008liblinear,yu2011dual,simon2013blockwise}.
Note that unlike the linear regression version where we can calculate a closed-form solution for $\bb$, 
here the penalty of $\bb$ is necessary or we will not achieve a unique solution,
\ie~the solution is ill-posed~\cite{tikhonov1977solutions}.
For example, assume that we have a large enough $\lambda_2$ to vanish all elements of $\bg$.
Then the problem degenerates to the normal logistic regression with the coefficient $\bb$.
Suppose we have an optimal solution $\bb^*$, 
and we replace the $k$-th row $\bb^*_{k,\cdot}$ by $\bb^*_{k,\cdot}+\varepsilon \bm{1}^\top$ where $\varepsilon$ is some scalar. 
Then we have
\begin{equation}
\hat{\bY}_{i,l\mid_{\bb^{*}_{k,\cdot}+\varepsilon \bm{1}^\top}}
=\frac{e^{\bm{X}_{i\cdot}\bb^{*}_{\cdot c}+x_{i,k}\varepsilon}}
{\sum_{l=1}^Ce^{\bm{X}_{i\cdot}\bm{\beta}^*_{\cdot l}+x_{i,k}\varepsilon}}
=\frac{e^{\bm{X}_{i\cdot}\bm{\beta}^*_{\cdot c}}}
{\sum_{l=1}^Ce^{\bm{X}_{i\cdot}\bm{\beta}^*_{\cdot l}}}
=\hat{\bY}_{i,l\mid_{\bb^{*}_{k,\cdot}}}
\end{equation}
Hence, to get a unique solution, we must provide some penalty on $\bb$.

We use a partial Newton algorithm~\cite{simon2013blockwise} to solve this optimization problem.
Similar to the linear regression model, we use a list of $\lambda$s to calculate the regularization path of $\bg$. 

\begin{algorithm}
\textbf{Input}:
support data$\left\{ \left(\bX_{i},\by_{i}\right)\right\} _{i=1}^{c\times s}$, query data $\bX_{t}=\left\{ \bX_{j}\right\} _{j=1}^{M}$, unlabeled data $\bX_{u}=\left\{ \bX_{k}\right\} _{k=1}^{U}$

\textbf{Initialization}: support set $\left(\bX_{s},\bY_{s}\right)=\left\{ \left(\bX_{i},\by_{i}\right)\right\} _{i=1}^{c\times s}$, feature matrix $\bX_{c\times s+U,d}=\left[\bX_{s};\bX_{u}\right]$, classifier

\textbf{Repeat:}

Train classifier using $\left(\bX_{s},\bY_{s}\right)$;

Get pseudo-label $\bY_u$ for $\bX_u$ by classifier;

Rank $\left(\bX,\bY\right)=\left(\bX,[\bY_s;\bY_u]\right)$ by ICI;

Select a subset $\left(\bX_{\mathrm{sub}},\bY_{\mathrm{sub}}\right)$ into $\left(\bX_{s},\bY_{s}\right)$;

\textbf{Until Converged.}

\textbf{Inference:}

Train classifier using $\left(\bX_{s},\bY_{s}\right)$;

Get pseudo-label $\bY_t$ for $\bX_t$ by classifier;

\textbf{Output}: inference labels $\bY_{t}=\left\{ \hat{\by}_{j}\right\} _{j=1}^{M}$

\caption{\label{alg:Inference-process.}Inference process of our algorithm.}
\end{algorithm}

\subsection{Self-taught learning with ICI}

The proposed ICI can thus be easily integrated to improve the self-taught learning algorithm. Particularly, the initialized classifier can predict the pseudo-labels of unlabeled instances; and we further employ the ICI algorithm to select the most confident subset of unlabeled instances, to update the classifier. The whole algorithm can be iteratively updated, as summarized in Algorithm~\ref{alg:Inference-process.}. 
We also show a qualitative result in an inference episode in Fig.~\ref{fig:qualitative-images}.

Intuitively, ICI focuses on fitting a line using the observations $\left(\bm{x}_{i},\bm{y}_{i}\right)_{i=1}^{n}$ which contains outliers.
\revise{Starting} from the labeled instances, we search the most possible inliers from the pseudo-labeled instances in each iteration.
When we solve the line along the regularization path (from $\lambda_{max}$ to $\lambda_{min}$), the estimated line will approach the more linear-separable subset, resulting in $\left\Vert\bg_i\right\Vert=0$ for instances in this subset while $\left\Vert\bg_i\right\Vert>0$ for others.
Then we could use the linear-separable subset to improve the linear classifier.
Furthermore, the fitted line cannot provide the right label for those outliers, hence the re-train process and re-infer process are essential to transfer outliers to inliers.

\section{Identifiability of ICI}
In this part, we provide a theory for identifiability of ICI with linear regression model.
Our theory is based on the model selection consistency for a linear regression with $\ell_1$-sparsity regularization ~\cite{zhao2006model,wainwright2009sharp}.
Here our purpose is to answer the question of 
\textit{under which conditions can we find the right-predicted instances}?

Recall that our intuition is that $\bg_{i,j}$ can be regarded as the correction of the chance that instance $i$ belonging to class $j$.
Suppose $\bg^*$ is the ground truth.
If the pseudo-labeled instance $i$ is right-predicted, 
then we have $\bg^*_{i,j}=0,\forall  j \in \left\{ 1,\ldots,c\right\}$.
On the contrary,
if the instance is wrongly predicted, then we should have
$\bg^*_{i,j}\neq0$ for some $j$.

We start with reformulating the derivation process from Eq.~\eqref{eq:loss_func} to Eq.~\eqref{eq:penalty} by another decoupled representation of solving $\bb$ and $\bg$.
\minor{
Recall that the linear regression model with incidental parameters is
\begin{equation}
\bm{Y}=\bm{X}\bm{\beta}^{*}+\bm{\gamma}^{*}+\bm{\varepsilon},
\end{equation}
where $\bm{Y}\in\left\{ 0,1\right\} ^{n\times c},\bm{X}\in\mathbb{R}^{n\times d},\bm{\beta}^{*}\in\mathbb{R}^{d\times c},\bm{\gamma}^{*}\in\mathbb{R}^{n\times c},\bm{\varepsilon}\in\mathbb{R}^{n\times c}$.
We are solving the problem of 
\begin{equation}
\underset{\bm{\beta},\bm{\gamma}}{\mathrm{argmin}}\frac{1}{2}\left\Vert \bm{Y}-\bm{X}\bm{\beta}-\bm{\gamma}\right\Vert _{\mathrm{F}}^{2}+\lambda\sum_{i=1}^{n}\sum_{j=1}^{c}\left|\gamma_{i,j}\right|.
\end{equation}
With this formulation, one could vectorize the problem and transfer it into the single-response regression case. 
Denote the vectorization operator for $\bm{A}\in\mathbb{R}^{m\times n}$ as $\mathrm{vec}\left(\bm{A}\right)\coloneqq\left(a_{1,1},\ldots,a_{m,1},a_{1,2},\ldots,a_{m,2},\ldots,a_{1,n},\ldots,a_{m,n}\right)^{\top}$,
then
\begin{equation}
\mathrm{vec}\left(\bm{Y}\right)=\left(\bm{I}_{c}\otimes\bm{X}\right)\mathrm{vec}\left(\bm{\beta}^{*}\right)+\mathrm{vec}\left(\bm{\gamma}^{*}\right)+\mathrm{vec}\left(\bm{\varepsilon}\right),
\end{equation}
where $\otimes$ is the Kronecker product operator. We denote $\vec{\bm{y}}=\mathrm{vec}\left(\bm{Y}\right)\in\left\{ 0,1\right\} ^{nc},\bm{X}_{\otimes}=\left(\bm{I}_{c}\otimes\bm{X}\right)\in\mathbb{R}^{nc\times dc},\vec{\bm{\beta}}=\mathrm{vec}\left(\bm{\beta}\right)\in\mathbb{R}^{dc},\vec{\bm{\gamma}}=\mathrm{vec}\left(\bm{\gamma}\right)\in\mathbb{R}^{nc},\vec{\bm{\varepsilon}}=\mathrm{vec}\left(\bm{\varepsilon}\right)\in\mathbb{R}^{nc}$.
We are now solving the problem of 
\begin{equation}
\underset{\vec{\bm{\beta}},\vec{\bm{\gamma}}}{\mathrm{argmin}}\frac{1}{2}\left\Vert \vec{\bm{y}}-\bm{X}_{\otimes}\vec{\bm{\beta}}-\vec{\bm{\gamma}}\right\Vert _{\mathrm{2}}^{2}+\lambda\left\Vert \vec{\bm{\gamma}}\right\Vert _{1}.
\end{equation}
}
We conduct the singular vector decomposition of $\bX_{\otimes}$ as $\bX_{\otimes}=\bm{U}\bm{\Sigma} \bm{V}^{\top}$,
where $\bm{U}\in\mathbb{R}^{nc\times nc},\ \bm{\Sigma}\in\mathbb{R}^{nc\times dc},\ \bm{V}\in\mathbb{R}^{dc\times dc}$.
Recall that $d$ is set as the reduced dimension from the original feature, hence we have $d\ll n$. 
Thus we could divide $\bm{U}$ into $\bm{U}=\left[\bm{U}_1,\bm{U}_2\right]$ where 
$\bm{U}_1$ is an orthogonal basis of the column space of $\bX_{\otimes}$.
Then we have $\bU^\top\bU=\bU\bU^\top=\bm{I}$ and $\bU_2^\top \bX_{\otimes}=0$.
Hence
\begin{equation}
\label{eq:thm-loss}
\begin{aligned}
L\coloneqq&\left\Vert\vec{\bm{y}}-\bX_{\otimes}\vec{\bm{\beta}}-\vec{\bg}\right\Vert _{2}^{2}
=\left\Vert\bU^\top\left(\vec{\bm{y}}-\bX_{\otimes}\vec{\bm{\beta}}-\vec{\bg}\right)\right\Vert _{2}^{2}\\
=&\left\Vert\bU_1^\top\vec{\bm{y}}-\bU_1^\top\bX_{\otimes}\vec{\bm{\beta}}-\bU_1^\top\vec{\bg}\right\Vert _{2}^{2}
+\left\Vert\bU_2^\top\vec{\bm{y}}-\bU_2^\top\vec{\bg}\right\Vert _{2}^{2}.
\end{aligned}
\end{equation}
Again, we find the derivative with respect to $\vec{\bb}$ and make it equal to 0,
then we have
\begin{equation}
\hat{\vec{\bb}}=\left(\bX_{\otimes}^{\top}\bX_{\otimes}\right)^{\dagger}\bX_{\otimes}^{\top}\left(\vec{\bm{y}}-\vec{\bg}\right).
\label{eq:thm-beta}
\end{equation}
\revise{
Note that since $\partial L/\partial\minor{\hat{\vec{\bm{\beta}}}}=0$, we have 
\begin{equation}
\bX_{\otimes}^{\top}\bm{U}_{1}\left(\bm{U}_{1}^{\top}\vec{\bm{y}}-\bm{U}_{1}^{\top}\bX_{\otimes}\minor{\hat{\vec{\bm{\beta}}}}-\bm{U}_{1}^{\top}\vec{\bg}\right)=0.
\end{equation}
Denote $\mathrm{rank}\left(\bX_{\otimes}\right)=k$, then we have $\bX_{\otimes}^{\top}\bm{U}_{1}\in\mathbb{R}^{dc\times k}$,
$\bm{U}_{1}^{\top}\vec{\bm{y}}-\bm{U}_{1}^{\top}\bX_{\otimes}\minor{\hat{\vec{\bm{\beta}}}}-\bm{U}_{1}^{\top}\vec{\bg}\in\mathbb{R}^{k\times 1}$
and $\mathrm{rank}\left(\bX_{\otimes}^{\top}\bm{U}_{1}\right)=k$ by definition. 
Using Sylvester\textquoteright s rank inequality, we have
\begin{equation}
\begin{aligned}
&\mathrm{rank}\left(\bX_{\otimes}^{\top}\bm{U}_{1}\right)+\mathrm{rank}\left(\bm{U}_{1}^{\top}\vec{\bm{y}}-\bm{U}_{1}^{\top}\bX_{\otimes}\minor{\hat{\vec{\bm{\beta}}}}-\bm{U}_{1}^{\top}\vec{\bg}\right)-k\\
\leq&\mathrm{rank}\left(\bX_{\otimes}^{\top}\bm{U}_{1}\left(\bm{U}_{1}^{\top}\vec{\bm{y}}-\bm{U}_{1}^{\top}\bX_{\otimes}\minor{\hat{\vec{\bm{\beta}}}}-\bm{U}_{1}^{\top}\vec{\bg}\right)\right)=0.
\end{aligned}
\end{equation}
Hence 
\begin{equation}
\mathrm{rank}\left(\bm{U}_{1}^{\top}\vec{\by}-\bm{U}_{1}^{\top}\bX_{\otimes}\minor{\hat{\vec{\bm{\beta}}}}-\bm{U}_{1}^{\top}\vec{\bg}\right)=0.
\end{equation}
Hence the first term of $L$ equals to $0$. 
Now we are solving the problem of
}
\begin{equation}
L\left(\vec{\bg}\right)=\left\Vert\bU_2^\top\vec{\bm{y}}-\bU_2^\top\vec{\bg}\right\Vert _{2}^{2}+\lambda \left\Vert \vec{\bm{\gamma}}\right\Vert _{1}.
\label{eq:another-penalty}
\end{equation}
Eq.~\eqref{eq:another-penalty} is equivalent to Eq.~\eqref{eq:penalty} but provides another interpretation that the incidental parameters (with a projection) try to find a sparse approximation of $\bU_2^\top\vec{\bm{y}}$.
Based on this, we could provide the answer of \textit{under which condition could we recover the true support set of $\vec{\bg}$}? 

Formally, let $S=\mathrm{supp}\left(\vec{\bg}^*\right)$ and $\hat{S}=\mathrm{supp}\left(\hat{\vec{\bg}}\right)$, 
where $\vec{\bg}^*$ is the 
\revise{ground-truth} prediction error,
$\hat{\vec{\bg}}$ is the estimator provided by our algorithm
\minor{
and $\mathrm{supp}\left(\vec{\bg}\right)=\{i\mid\vec{\bg}_{i}\neq 0\}$.
Recall that our goal is to find the wrongly predicted instances.
Hence we further define a ground-truth wrongly-predicted set $O=\left\{ i\vert\gamma_{i,j}^{*}\neq0,\textrm{ for some }j\in\left[c\right]\right\}$ and the estimator  $\hat{O}=\left\{ i\vert\hat{\gamma}_{i,j}\neq0,\textrm{ for some }j\in\left[c\right]\right\}$.
}
For simplicity,
\revise{
we denote $\vec{\bm{y}}_{u}=\bU_2^\top \vec{\bm{y}}$ and $\tilde{\bU}=\bU_2^\top$.
}
Furthermore,
denote $\tilde{\bU}_S$ ($\tilde{\bU}_{S^c}$) as the column vectors of $\tilde{\bU}$ whose index are in $S$ ($S^c$), 
respectively.
We are solving the problem of 
\begin{equation}
\label{eq:thm-problem}
\min_{\vec{\bg}} \left\Vert\vec{\bm{y}}_{u}-\tilde{\bU}\vec{\bg}\right\Vert _{2}^{2}+\lambda \left\Vert \vec{\bm{\gamma}}\right\Vert _{1},
\end{equation}
\minor{
Recall that the linear regression model indicates that 
for ground-truth values $\vec{\bm{\beta}}^{*},\vec{\bm{\gamma}^{*}}$
\begin{equation}
\vec{\bm{y}}=\bm{X}_{\otimes}\vec{\bm{\beta}}^{*}+\vec{\bm{\gamma}}^{*}+\vec{\bm{\varepsilon}},
\end{equation}
and hence
\begin{equation}
\tilde{\bU}\vec{\bm{y}}=\tilde{\bU}\left(\bm{X}_{\otimes}\vec{\bm{\beta}}^{*}+\vec{\bm{\gamma}}^{*}+\vec{\bm{\varepsilon}}\right).
\end{equation}
Hence we have 
\begin{equation} 
\vec{\bm{y}}_{u}=\tilde{\bU}\vec{\bm{y}}=\tilde{\bU}\vec{\bm{\gamma}}^{*}+\tilde{\bU}\vec{\bm{\varepsilon}}=\tilde{\bm{U}}_{S}\vec{\bm{\gamma}}_{S}^{*}+\tilde{\bU}\vec{\bm{\varepsilon}},\label{eq:y-ground-truth}
\end{equation}
where $\vec{\bm{\varepsilon}}$ is the sub-Gaussian noise assumed in the
linear regression model.
}
Further let $\mu_{\tilde{\bU}}=\underset{i\in S^{c}}{\max}\left\Vert \tilde{\bU}_{i}\right\Vert _{2}^{2}$.
We give three assumptions:

\noindent(C1: Restricted eigenvalue)
\begin{equation}
\lambda_{\min }\left(\tilde{\bU}_{S}^{\top} \tilde{\bU}_{S}\right)=C_{\min }>0.
\end{equation}
(C2: Irrepresentability) $\exists\ \eta\in\left(0,1\right]$,
\begin{equation}
\left\|\tilde{\bU}_{S^{c}}^{\top} \tilde{\bU}_{S}\left(\tilde{\bU}_{S}^{\top} \tilde{\bU}_{S}\right)^{-1}  \right\|_{\infty} \leq 1-\eta.
\end{equation}
(C3: Large error)
\begin{equation}
\vec{\bg}_{\min }:=\min _{i \in S}\left|\vec{\bg}_{i}^{*}\right|>  h\left(\lambda, \eta, \tilde{\bU}, \vec{\bg}^{*}\right),
\end{equation}
where
\begin{equation}
h\left(\lambda, \eta, \tilde{\bU}, \vec{\bg}^{*}\right)=\frac{\lambda\eta}{\sqrt{C_{\min } \mu_{\tilde{\bU}}}}+\lambda\left\|\left(\tilde{\bU}_{S}^{\top} \tilde{\bU}_{S}\right)^{-1} \operatorname{sign}\left(\vec{\bg}_{S}^{*}\right)\right\|_{\infty}
\end{equation}
\revise{and $\left\Vert \bm{A}\right\Vert_{\infty}\coloneqq\max_{i}\sum_j\left|A_{i,j}\right|$.}
Based on these conditions, we could provide the following theorem:
\begin{thm}[\minor{Identifiability of ICI}]
\label{thm:sufficiency}
Let 
\begin{equation*}
\lambda \geq \frac{2 \sigma \sqrt{\mu_{\tilde{\bU}}}}{\eta } \sqrt{ \log cn}.
\end{equation*}
Then with probability greater than
\begin{equation*}
1-2 cn \exp \left\{-\frac{\lambda^{2} \eta^{2}}{2  \sigma^{2} \mu_{\tilde{\bU}}}\right\} \geq 1-2 \left(cn\right)^{-1},
\end{equation*}
Eq.~\eqref{eq:thm-problem} has a unique solution $\hat{\bg}$ satisfies the following properties:
\begin{enumerate}
\item 
If C1 and C2 hold, the wrong-predicted instances indicated by ICI has no false positive error, i.e. 
$\hat{S}\subseteq S$
\minor{and hence $\hat{O}\subseteq O$}
, and
\begin{equation*}
\left\|\hat{\vec{\bg}}_{S}-\vec{\bg}_{S}^{*}\right\|_{\infty} \leq  h\left(\lambda, \eta, \tilde{\bU}, \vec{\bg}^{*}\right);
\end{equation*}
\item
If C1, C2, and C3 hold, ICI will identify all the correctly-predicted instances, i.e. $\hat{S}= S$ \minor{and hence $\hat{O} = O$}~(in fact~$\mathrm{sign} \left(\hat{\vec{\bg}}\right)=\mathrm{sign} \left(\vec{\bg}^*\right)$).
\end{enumerate}
\end{thm}
\minor{
\begin{rem*}
Assumption C1 is necessary to ensure that there is a unique $\vec{\bg}^*$ satisfying model \eqref{eq:y-ground-truth}. Assumptions C1-C2 (C1-C3) are sufficient for $\hat{O}\subseteq O$ ($\hat{O} = O$), respectively. They are also necessary in the sense that once violated, there are cases which fail the conclusion with non-vanishing probability.
\end{rem*}
}

The proof is given in the Appendix section \ref{appendix}.
The theorem shows that our algorithm could find the right-predicted pseudo-labeled instances under specific conditions.
Practically,
it may be hard for us to choose a reasonable $\lambda$  to satisfy the three conditions since we could not know $\vec{\bg}_S^*$ in advance.
Specifically, in  the tasks of both semi-supervised and transductive few-shot learning concerned in this paper, one can not assume knowing $\vec{\bg}_S^*$.
Hence, we use the iterative strategy to search along the solution path to select the instances automatically.

\mypar{Effectiveness of the identifiablity in reality.}
\minor{
It is desirable to check to which extent the assumptions hold in reality.
To answer this question, we run 5-way-1-shot TFSL experiments on \emph{mini}ImageNet dataset for 2000 episodes.
}
\begin{figure}[h]
\centering
\includegraphics[width=0.6\columnwidth]{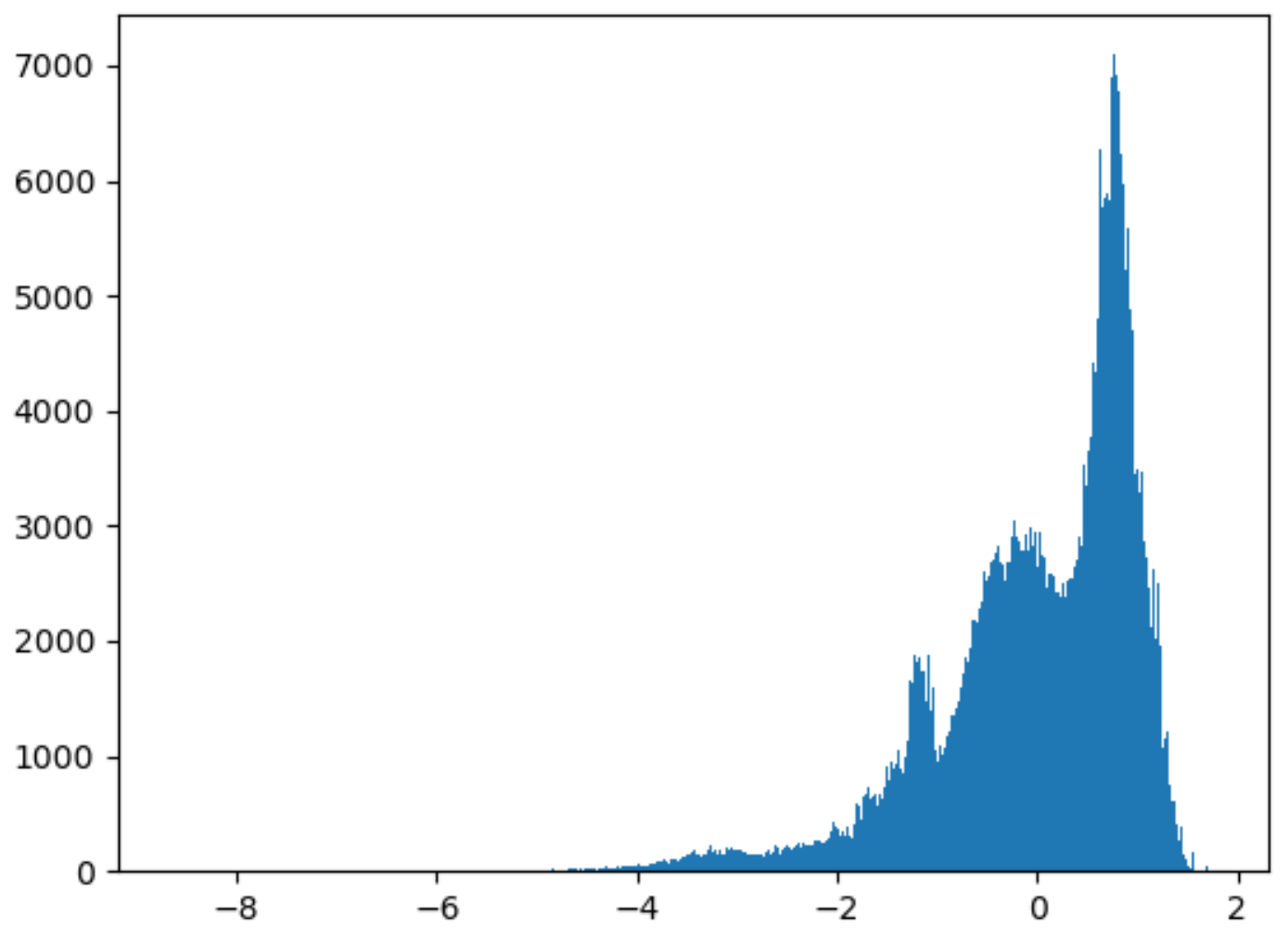} 
\caption{\label{fig:hist-of-error}
Histogram of errors in 2000 episodes. 
The x-axis is the value of errors, while the y-axis is the number of errors.
}
\end{figure}

\mypar{Sub-Gaussian noise.}
\minor{
We collect all the noises in the 2000 episodes and visualize the histogram in Fig.~\ref{fig:hist-of-error}. 
It can be seen that the noise can be approximated by a Gaussian Mixed Model, specifically the sum of three independent Gaussian distribution. 
Hence the noise can be assumed as following sub-Gaussian distribution with bounded variance.
Further, the magnitude of sample mean of the noises is $10^{-19}$, which can be seen as zero mean.
}
\begin{table}[H]
\begin{centering}
\begin{tabular*}{\columnwidth}{@{\extracolsep{\fill}}lcccc}
\toprule 
Satisfied Assumptions & None & C1 & C1 and C2 & All\tabularnewline
\midrule
Improved Episodes & $0$ & $424$ & $1035$ & $40$\tabularnewline
Total Episodes & $0$ & $793$ & $1164$ & $43$\tabularnewline
I/T & $-$ & $53.5\%$ & $88.9\%$ & $93.0\%$\tabularnewline
\bottomrule
\end{tabular*}
\par\end{centering}
\caption{\label{tab:Number-of-episodes}Number of episodes satisfying each
assumption and whether the transductive inference improve the performance.}
\end{table}
\mypar{Assumptions C1-C3.}
\minor{
In each episode, we test whether the assumptions are satisfied and count them in Table \ref{tab:Number-of-episodes}.
We can see that: 
(i) In more than half of the episodes the assumptions C1-C2 are satisfied.
From our theorem, in this case ICI will have no false positive error.
Hence our ICI will reduce the noise of pseudo-labeled instances without eliminating the correctly-predicted instances.
Practically, most of them ($\left(1035+40\right)/\left(1164+43\right)=89.0\%$)
will achieve better performance after transductive inference. 
(ii) When all the assumptions are satisfied, the transductive inference
will get better performance in a high ratio ($93.0\%$).
(iii) Even if C2-C3 are not satisfied, transductive
inference still have the chance of improving the performance ($53.5\%$).
One major reason is that our iterative update strategy will help reduce
the noise.
}

\section{Experiments}
\mypar{Datasets.}
Our experiments are conducted on \revise{four} widely \revise{used} few-shot learning benchmark datasets 
including
\emph{mini}ImageNet~\cite{ravi2016optimization}, \emph{tiered}ImageNet~\cite{ren2018meta}, 
CIFAR-FS~\cite{bertinetto2018metalearning} and 
CUB~\cite{wah2011caltech}.
\textbf{\emph{mini}}\textbf{ImageNet}\footnote{ \revise{https://github.com/gidariss/FewShotWithoutForgetting}} consists of $100$ classes with $600$ labeled instances \revise{per} category.
We follow the split proposed by~\cite{ravi2016optimization}, using $64$ classes as the base set to train the feature extractor, $16$ classes as the validation set, and 
report performance on the novel set which consists of $20$ classes.
\textbf{\emph{tiered}}\textbf{ImageNet}\footnote{\revise{https://github.com/yaoyao-liu/meta-transfer-learning}} is a larger dataset compared \revise{to} \emph{mini}ImageNet, and its categories are selected \revise{from a} hierarchical structure to split base and novel datasets semantically. 
We follow the split introduced in~\cite{ren2018meta} with base set of $20$ superclasses ($351$ classes), validation set of $6$ superclasses ($97$ classes) and novel set of $8$ superclasses ($160$ classes). 
Each class contains $1281$ images on average. 
\textbf{CUB}\footnote{\revise{http://www.vision.caltech.edu/visipedia/CUB-200-2011.html}} is a fine-grained dataset of $200$ bird categories with $11788$ images in total. 
Following the previous \revise{few-shot} setting in~\cite{hilliard2018few}, we use \revise{$100$, $50$ and $20$ classes for base, validation and novel set respectively.}
To make a fair comparison \revise{in model training and testing}, we crop the bounding \revise{boxes} provided by~\cite{triantafillou2017few} \revise{for all the images in CUB.}
\textbf{CIFAR-FS}\footnote{\revise{https://github.com/bertinetto/r2d2}} is a dataset derived from CIFAR-100~\cite{krizhevsky2009learning} \revise{with lower-resolution images.}
It contains $100$ classes with $600$ instances in each class. 
We follow the common split given by~\cite{bertinetto2018metalearning}, using $64$ classes to construct the base set, $16$ for validation, and $20$ as the novel set.

\mypar{Experimental setup.}
\revise{We present the implementation details and experiment settings in the following.}
\revise{Unless otherwise specified, our implementation details and experiment setting are same with the default setting adopt by majority few-shot learning methods~\cite{ye2020fewshot, Liu2020E3BM,Zhang_2020_CVPR,lee2019meta,hilliard2018few} for a fair comparison.}
\revise{Same as}~\cite{oreshkin2018tadam,lee2019meta}, we \revise{employ} ResNet-12~\cite{DBLP:journals/corr/HeZRS15} with $4$ residual blocks as the feature extractor in our experiments. 
Each \revise{residual} block consists of three $3\times3$ convolutional layers, each of which followed by a \revise{batch normlization} layer and a LeakyReLu~(0.1) activation. \revise{A $2\times2$ max-pooling layer is appended at the end of each block to downsample the spatial size.}
The number of filters in each block is $64$, $128$, $256$ and $512$ respectively.
Specifically, \revise{following}~\cite{lee2019meta}, we adopt the \textit{Dropout}~\cite{JMLR:v15:srivastava14a} in first two blocks to vanish $10\%$ of the output,
and adopt \textit{DropBlock}~\cite{ghiasi2018dropblock} in latter two blocks to vanish $10\%$ of output at channel level.
Finally, an average-pooling layer is employed to produce the input feature embedding.
\revise{
We use the baseline method R12-proto-ac introduced in~\cite{CAN} to train the backbone with the global and nearest neighbor classification loss.
}
\revise{SGD with momentum is adopted} as the optimizer to train the feature extractor \textit{from scratch}.
Momentum factor and strength of $L_{2}$ weight decay is set to $0.9$ and $5e-4$, respectively.
All \revise{input images} are resized to $84\times84$.
\revise{
Our initial learning rate is set to $0.1$ and decay to $0.006,~0.0012$ and $0.00024$ after $60,~70$ and $80$ epochs, respectively.
}
The total training epochs \revise{is set to $90$}. 
In all of our experiments, we normalize the feature with $L_2$ norm and reduce the feature dimension to $d=5$ using LLE~\cite{roweis2000nonlinear}
\revise{for the pre-processing part of ICI, while the classification part still use the original features}.
\revise{We use the logistic regression as our basic classifier.}
Our model and all baselines are evaluated over $2000$ episodes with $15$ test samples in each class. 

\begin{table*}[!ht]
\centering
\begin{tabular*}{\textwidth}{@{\extracolsep{\fill}} p{0.7cm} l l l l l l l l l}
\toprule 
\multirow{2}{*}{Setting} & \multirow{2}{*}{Model} &  \multicolumn{2}{c}{\emph{mini}ImageNet} & \multicolumn{2}{c}{\emph{tiered}ImageNet} & \multicolumn{2}{c}{CIFAR-FS} & \multicolumn{2}{c}{CUB}\tabularnewline
 & &$1$shot & $5$shot & $1$shot & $5$shot & $1$shot & $5$shot & $1$shot & $5$shot\tabularnewline
\midrule
\multirow{11}{*}{In.}
&Baseline$^{*}$~\cite{DBLP:journals/corr/abs-1904-04232} & $51.75$\ci{0.80} & $74.27$\ci{0.63} & - & - & - & - & $65.51$\ci{0.87} & $82.85$\ci{0.55}\tabularnewline
& Baseline++$^{*}$~\cite{DBLP:journals/corr/abs-1904-04232}&$51.87$\ci{0.77} &$75.68$\ci{0.63} & - & - & - & - & $67.02$\ci{0.90} & $83.58$\ci{0.54}\tabularnewline
& MatchingNet$^{*}$~\cite{vinyals2016matching}& $52.91^{\textcolor{black}{1}}$\ci{0.88} & $68.88^{\textcolor{black}{1}}$\ci{0.69} & - & - & - & - & $72.36^{\textcolor{black}{1}}$\ci{0.90} & $83.64^{\textcolor{black}{1}}$\ci{0.60}\tabularnewline
& ProtoNet$^{*}$~\cite{snell2017prototypical}& $54.16^{\textcolor{black}{1}}$\ci{0.82} & $73.68^{\textcolor{black}{1}}$\ci{0.65} & - & - & $72.20^{\textcolor{black}{3}}$ & $83.50^{\textcolor{black}{3}}$ & $71.88^{\textcolor{black}{1}}$\ci{0.91} & $87.42^{\textcolor{black}{1}}$\ci{0.48}\tabularnewline
& MAML$^{*}$~\cite{finn2017model}& $49.61^{\textcolor{black}{1}}$\ci{0.92}& $65.72^{\textcolor{black}{1}}$\ci{0.77} & - & - & - & - & $69.96^{\textcolor{black}{1}}$\ci{1.01} & $82.70^{\textcolor{black}{1}}$\ci{0.65}\tabularnewline
&RelationNet$^{*}$~\cite{sung2018learning} & $52.48^{\textcolor{black}{1}}$\ci{0.86} & $69.83^{\textcolor{black}{1}}$\ci{0.68} & - & - & - & - & $67.59^{\textcolor{black}{1}}$\ci{1.02} & $82.75^{\textcolor{black}{1}}$\ci{0.58}\tabularnewline
& adaResNet~\cite{munkhdalai2018rapid}& $56.88$ & $71.94$ & - & - & - & - & - & -\tabularnewline
& TapNet~\cite{yoon2019tapnet} & $61.65$ & $76.36$ & $63.08$& $80.26$  & - & - & - & -\tabularnewline
& CTM$^{\dag}$~\cite{li2019finding} & $64.12$ & $80.51$ & $68.41$ & $84.28$  & - & - & - & -\tabularnewline
&MetaOptNet~\cite{lee2019meta}&$64.09$&$80.00$&$65.81$&$81.75$&$72.60$&$84.30$&-&-\tabularnewline
\midrule

\multirow{4}{*}{Tran.}
&TPN~\cite{liu2018learning} & $59.46$ & $75.65$ & $58.68^{\textcolor{black}{4}}$ & $74.26^{\textcolor{black}{4}}$ & $65.89^{\textcolor{black}{4}}$ & $79.38^{\textcolor{black}{4}}$ & - & -\tabularnewline
&TEAM$^{*}$~\cite{qiao2019transductive}  & $60.07$ & $75.90$ & - & - & $70.43$ & $81.25$ & $80.16$ & $87.17$ \tabularnewline
&CAN+T~\cite{hou2019cross} & $67.19$\ci{0.55} & $80.64$\ci{0.35} & $73.21$\ci{0.58} & $84.93$\ci{0.38} & - & - & - & - \tabularnewline
&DPGN~\cite{yang2020dpgn} & $67.77$\ci{0.32} & $\textbf{84.60}$\ci{0.43} & $72.45$\ci{0.51} & $\textbf{87.24}$\ci{0.39} & $77.90$\ci{0.50} & $\textbf{90.20}$\ci{0.40} & $75.71$\ci{0.47} & $91.48$\ci{0.33} \tabularnewline
\midrule
\multirow{5}{*}{Semi.}
&MSkM + MTL & $62.10^{\textcolor{black}{2}}$ & $73.60^{\textcolor{black}{2}}$ & $68.6^{\textcolor{black}{2}}$ & $81.00^{\textcolor{black}{2}}$  & - & - & - &- \tabularnewline
&TPN + MTL & $62.70^{\textcolor{black}{2}}$ & $74.20^{\textcolor{black}{2}}$ & $72.10^{\textcolor{black}{2}}$ & $83.30^{\textcolor{black}{2}}$ & - & - & - & -\tabularnewline
&MSkM~\cite{ren2018meta}&$50.40$ & $64.40$ & $52.40$ & $69.90$ & - & - & - &  - \tabularnewline
&TPN~\cite{liu2018learning}& $52.78$& $66.42$ & $55.70$ & $71.00$ & - & - & - & -  \tabularnewline
&LST~\cite{sun2019learning}& $70.10$& $78.70$ & $77.70$ & $85.20$ & - & - & - & -  \tabularnewline
\midrule 
\midrule 

\multirow{2}{*}{Tran.}
&ICIC & $71.29$\ci{0.59} & $83.12$\ci{0.33} & $76.13$\ci{0.62} & $86.73$\ci{0.36}  & $78.47$\ci{0.60} & $86.41$\ci{0.36} & $90.38$\ci{0.42} & $94.30$\ci{0.20}\tabularnewline
&ICIR & $\textit{72.39}$\ci{0.62} & $83.27$\ci{0.33} & $77.48$\ci{0.62} & $86.84$\ci{0.36} & $79.19$\ci{0.63} & $86.66$\ci{0.36} & $ 90.89$\ci{0.43} & $94.36$\ci{0.20}\tabularnewline

\midrule 
\multirow{2}{*}{
\parbox[t]{0.7cm}{
Semi.  \\
15/15}}
&ICIC & $70.97$\ci{0.56} & $82.69$\ci{0.33} & $76.00$\ci{0.60} & $86.19$\ci{0.36}  & $78.44$\ci{0.58} & $86.10$\ci{0.36} & $89.89$\ci{0.42} & $94.00$\ci{0.20} \tabularnewline
&ICIR &  $72.32$\ci{0.58} & $82.78$\ci{0.33} & $76.98$\ci{0.61} & $86.24$\ci{0.36} & $79.20$\ci{0.58} & $86.14$\ci{0.36} & $90.45$\ci{0.42} & $94.00$\ci{0.20} \tabularnewline

\midrule
\multirow{2}{*}{\parbox[t]{0.7cm}{
Semi.  \\
30/50}}
&ICIC& $71.43$\ci{0.62} & $\textit{83.41}$\ci{0.35} & $\textit{78.01}$\ci{0.63} & $\textit{86.86}$\ci{0.37}  & $\textit{80.25}$\ci{0.58} & $86.99$\ci{0.36} & $\textit{91.75}$\ci{0.39} & $\textit{94.42}$\ci{0.20}\tabularnewline
&ICIR&  $\textbf{73.12}$\ci{0.65} & $83.28$\ci{0.37} & $\textbf{78.99}$\ci{0.66} & $86.76$\ci{0.39} & $\textbf{80.74}$\ci{0.61} & $\textit{87.16}$\ci{0.36} & $\textbf{92.12}$\ci{0.40} & $\textbf{94.52}$\ci{0.20} \tabularnewline

\bottomrule
\end{tabular*}
\caption{\label{fig:tfsl results}
The averaged accuracies with $95\%$ confidence intervals over $2000$ episodes on several datasets.  
Results with $\left(\cdot\right)^{1}$ are reported in~\cite{DBLP:journals/corr/abs-1904-04232}, 
with $\left(\cdot\right)^{\textcolor{black}{2}}$ are reported in~\cite{sun2019learning}, 
with $\left(\cdot\right)^{\textcolor{black}{3}}$ are reported in~\cite{lee2019meta}.
$\left(\cdot\right)^{\textcolor{black}{4}}$ is our implementation with the official code of~\cite{liu2018learning}. 
Methods denoted by $\left(\cdot\right)^*$ denotes ResNet-18 with input size $224\times224$, while $\left(\cdot\right)^{\dag}$ denotes ResNet-18 with input size $84\times84$. 
Our method and other alternatives use ResNet-12 with input size $84\times84$.
\textbf{In.} and \textbf{Tran.} indicate inductive and transductive setting, respectively. 
\textbf{Semi.} denotes semi-supervised setting where $(\cdot/\cdot)$ shows the number of unlabeled data available in $1$-shot and $5$-shot experiments.
ICIC indicates the logistic regression version of our model, 
and ICIR indicates the linear regression version.
We use logistic regression as our classifier.
\revise{
In each column, the highest result is in bold, and the second highest result is in italics.
}
}
\end{table*}
  
\subsection{Semi-supervised few-shot learning}
\mypar{Settings.} 
In the inference \revise{stage}, the unlabeled data from the corresponding category pool is utilized to help FSL. 
In our experiments, we report the following settings of SSFSL: 
(1) we use $15$ unlabeled samples for each class, the same as TFSL, 
\revise{to compare the performance of ICI between SSFSL and TFSL setting with the same number of unlabeled data.}
(2) we use $30$ unlabeled samples in $1$-shot task, and $50$ unlabeled samples in $5$-shot task, same as current SSFSL approaches~\cite{sun2019learning}; 
We denote these as 15/15 and 30/50 in Table~\ref{fig:tfsl results}. 
Note that CUB is a fine-grained dataset and does not have sufficient samples in each class, so we simply choose $5$ as support set, $15$ as query set and \revise{left} samples as unlabeled set (about $39$ samples on average) in the $5$-shot task in the latter setting. 
For all settings, we select $5$ samples for each class in each iteration. The process is finished when at most 15/15, 25/45 unlabeled instances are selected in total, respectively. 

\mypar{Competitors.} 
We compare our algorithm with \revise{existing} approaches in \revise{the SSFSL setting}. 
\revise{
TPN~\cite{liu2018learning} classifies query samples by propagating labels from the support set and extra unlabeled set.}
LST~\cite{sun2019learning} also uses self-taught learning strategy to pseudo-label data and select confident ones, but they \revise{achieve so by episodically training a neural network for many iterations.}
Other approaches include Masked Soft k-Means~\cite{ren2018meta} and a combination of MTL with TPN and Masked Soft k-Means reported by LST.

\mypar{Results.} The results are shown in Table~\ref{fig:tfsl results} where denoted as Semi. in the first column. 
\revise{We can observe} that:
(1) Comparing SSFSL with TFSL with the same number of unlabeled data, we can see that our SSFSL results are only reduced by a little or even beat TFSL results, which indicates that the information we got from the unlabeled data are robust and we can indeed handle the true distribution with unlabeled data practically.
(2) The more unlabeled data we get, the better performance we have. Thus we can learn more knowledge with more unlabeled data almost consistently using a linear classifier (\eg logistic regression). 
(3) \revise{Comparing} to other SSFSL approaches, ICI also achieves varying degrees of improvements in almost all tasks and datasets. These results further \revise{verify the effectiveness of our approach.} 

\subsection{Transductive few-shot learning}
\mypar{Settings.} 
In transductive few-shot learning setting, \revise{people} have the chance to access \revise{many} query data \revise{in one go} in the inference stage. 
Thus the unlabeled set and the query dataset are the same. 
In our experiments, we select $5$ instances for each class in each iteration and repeat our algorithm until all the query samples are included.

\mypar{Competitors.}
We compare ICI with current TFSL approaches. 
TPN~\cite{liu2018learning} constructs a graph and uses label propagation to transfer labels from support samples to query samples and learn their framework in a meta-learning way.
TEAM~\cite{qiao2019transductive} utilizes class prototypes with a data-dependent metric to inference labels of query samples.
\revise{
CAN+T~\cite{hou2019cross} uses the self-taught learning to train the model repeatedly within the specific designed network.
DPGN~\cite{yang2020dpgn} adopts contrastive comparisons to produce distribution representation.
}

\mypar{Results.} The results are shown in Table~\ref{fig:tfsl results} where denoted as Tran. in the first column. 
\revise{
Compared with current TFSL approaches, ICI is competitive, especially in the 1-shot tasks.
Importantly and theoretically, under mild conditions of restricted
eigenvalue, irrepresentability, and large error, we empirically show that our approach is guaranteed to collect  the correctly-predicted pseudo-labeled instances from the noisy pseudo-labeled set; and our ICIR results achieve very competitive performance in almost all dataset. Essentially, our algorithm is theoretically grounded,  orthogonal and useful to the other state-of-the-art methods. It is thus a future work of exploring how to incorporate our algorithm with the other competitors.
}

\subsection{Ablation study\label{subsec:Ablation-Study}}

\mypar{\revise{Visualization.}}
We visualize the regularization path of $\gamma$ in one episode of the inference process in Fig.~\ref{fig:effective} where red lines are instances that are correct-predicted while black lines are wrong-predicted ones. 
It is obvious that that most of the correct-predicted instances lie in the lower-left part. 
Since ICI select samples whose norm will vanish in a lower $\lambda$, so could get more correct-predicted instances than wrong-predicted instances in a high ratio. 
\begin{figure}[ht]
\begin{centering}
\includegraphics[width=0.8\columnwidth]{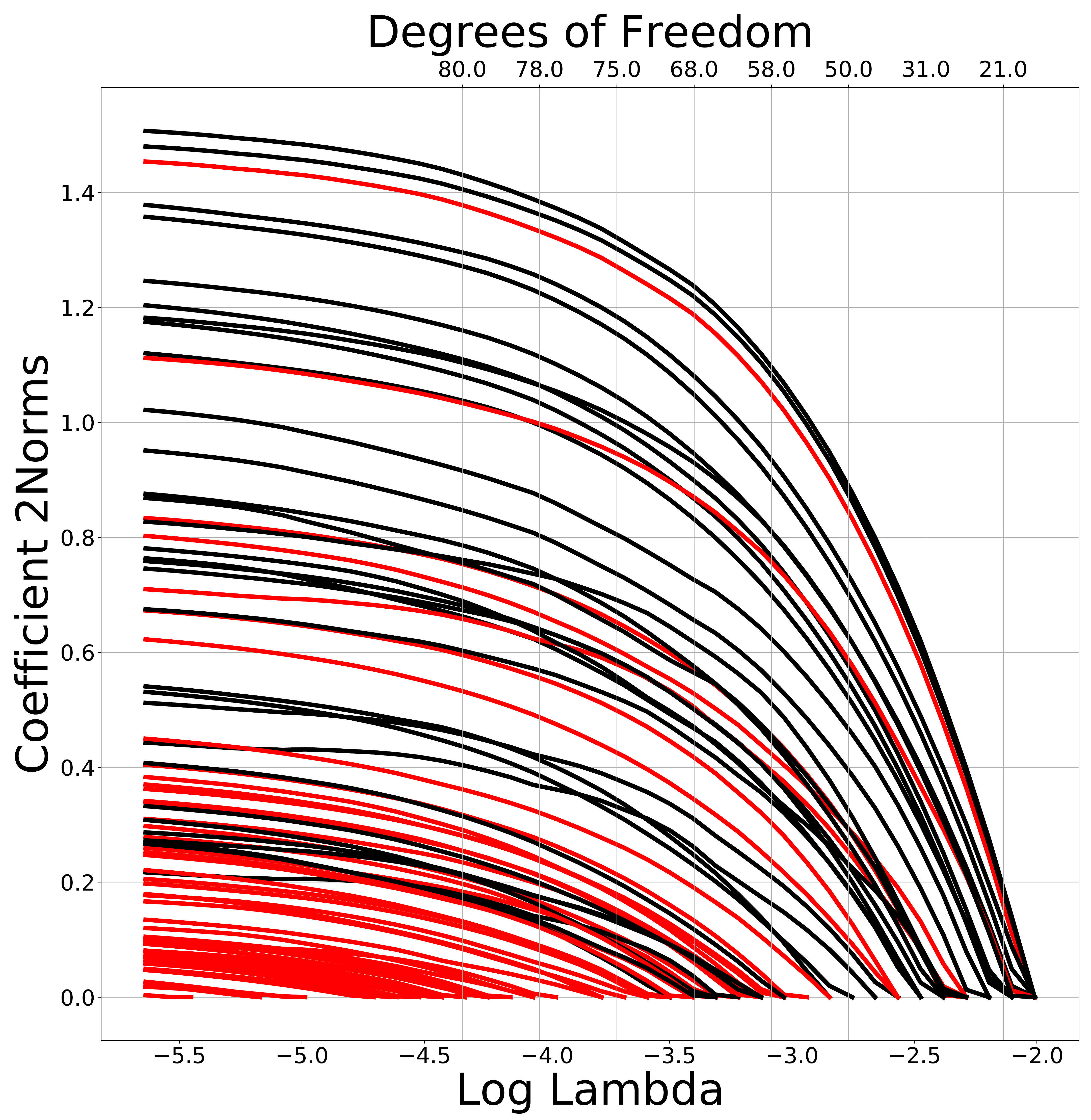}
\caption{\label{fig:effective}Regularization path of $\lambda$. Red lines are correct-predicted instances while black lines are wrong-predicted ones. ICI will choose instances in the lower-left subset.}
\end{centering}
\end{figure}

\mypar{\revise{Comparison with baselines.}}
To further show the effectiveness of ICI, 
we compare ICI with other sample selection strategies under the self-taught learning pipeline. 
\revise{
We consider the following baselines:
(1) RA (random): Select instances randomly.
(2) NN (nearest-neighbor): Select instances based on the distance between the pseudo-labeled
instances and the labeled instance. We will select the pseudo-labeled
instances which are the nearest neighbors of labeled instances with
the same (pseudo-)category.
(3) CO (confidence): Select instances based on the confidence given by the classifier,
where the confidence is defined as the  prediction scores/probabilities of the classifier. 
(4) CN (coefficient norm): Select instances based on the proposed metric
without considering the effect of $\gamma$. 
That is, selecting instances based on the y-axis in Fig.~\ref{fig:effective} instead of x-axis.
In this part, we have $15$ unlabeled instances for each class and select $5$ to re-train the classifier by different methods for Semi. and Tran. task on \emph{ mini}ImageNet. 
From Table~\ref{tab:ablation},
we observe that ICI outperforms all the baselines in all settings. }
\begin{table}[h]
\centering
\begin{tabular*}{\columnwidth}{@{\extracolsep{\fill}}lcccc}
\toprule 
\multirow{2}{*}{Model}&\multicolumn{2}{c}{Tran.}&\multicolumn{2}{c}{Semi.}\tabularnewline
&1shot&5shot&1shot&5shot\tabularnewline
\midrule
RA&$67.54$\ci{0.51}&$81.45$\ci{0.32}&$68.09$\ci{0.52}&$81.30$\ci{0.33}\tabularnewline
NN&$69.80$\ci{0.53}&$82.12$\ci{0.32}&$69.99$\ci{0.52}&$81.96$\ci{0.33}\tabularnewline
CO&$70.57$\ci{0.54}&$82.41$\ci{0.31}&$70.53$\ci{0.52}&$82.10$\ci{0.32}\tabularnewline
CN&$67.44$\ci{0.53}&$81.44$\ci{0.33}&$67.87$\ci{0.52}&$81.49$\ci{0.34}\tabularnewline
 \midrule
ICIR&  $\bf71.19$\ci{0.58} & $\bf82.55$\ci{0.32} & $\bf71.25 $\ci{0.55} & $\bf82.32$\ci{0.32}
\tabularnewline
\bottomrule
\end{tabular*}
\caption{\label{tab:ablation}
Compare to baselines on \emph{ mini}ImageNet under several settings.}
\end{table}
\revise{
The main reason why the confidence predicted by the classifier (i.e. the results of ``CO" in Table~\ref{tab:ablation})  is
not enough is that some high-confident predictions are actually wrongly-predicted.
Take the baseline of coefficient norm (CN) for example, the norm of the coefficient is directly the confidence score provided by the linear regression ``classifier'', where small norm indicates small error on fitting the corresponding sample.
In our illustration
of regularization path (see Fig.~\ref{fig:effective}),
the norm of some wrongly-predicted instances (see the lowest black
line for example) vanishes slower than the right-predicted instances.
This is a case when the confidence predicted by the classifier cannot
exclude the noise but ICI still works very well.
Particularly,  
the most important difference is that our x-axis method is theoretically guaranteed; in contrast, there is no theoretical guarantee for y-axis method and other sample selection baselines, as explained in Theorem~\eqref{thm:sufficiency}.
}

\begin{figure}[h]
\centering
\includegraphics[width=0.8\columnwidth]{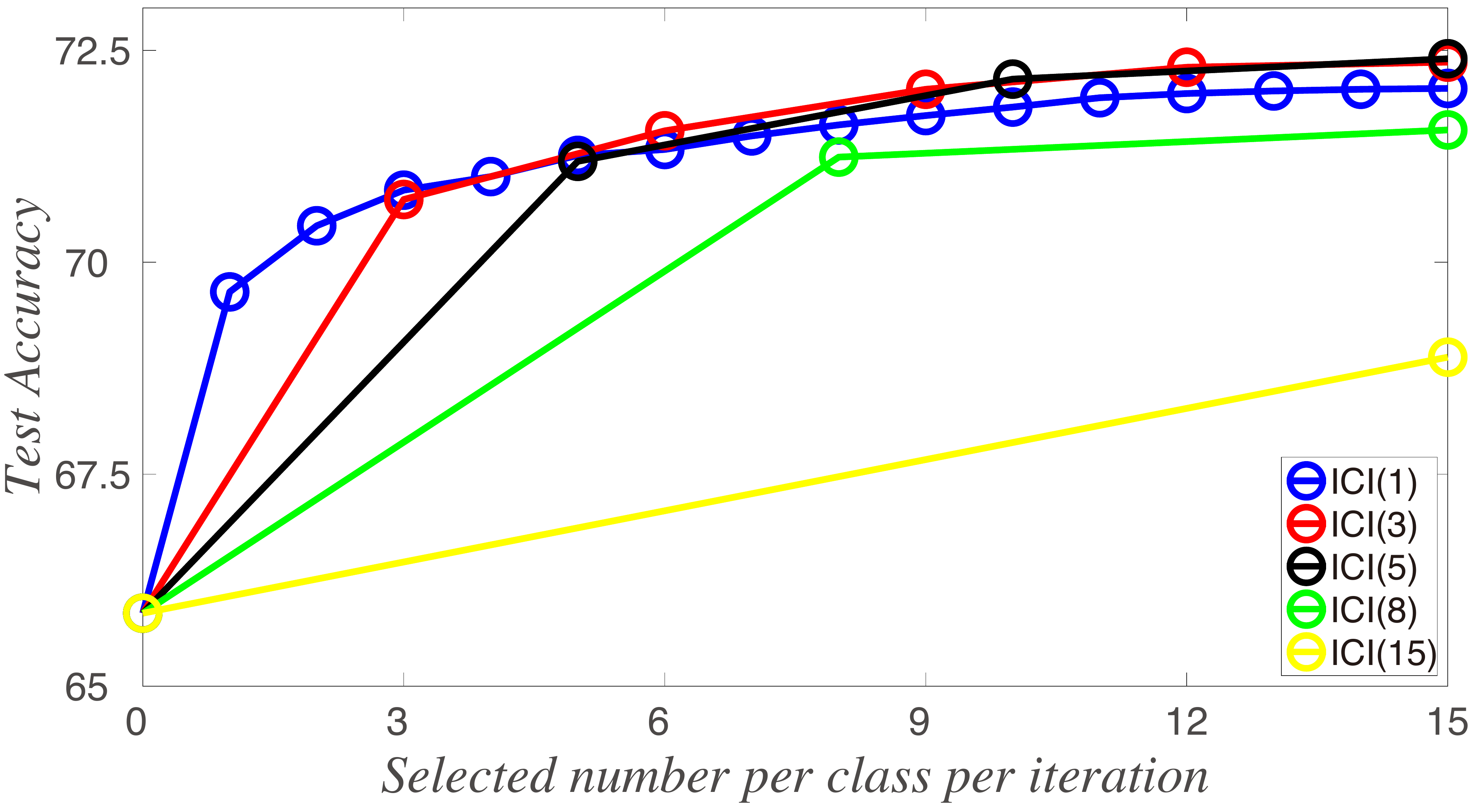} 
\caption{\label{fig:iter-manner}
Variation of accuracy as the selected samples increases over 2000 episodes on \emph{mini}ImageNet. 
``ICI (\textit{n})'': select \textit{n} samples per class in each iteration.}
\end{figure}
\mypar{Effectiveness of iterative manner.}
Our intuition is the proposed ICI learns to generate a set of trustworthy unlabelled data for classifier training. 
\revise{
One basic baseline is simply running the algorithm for one time, selecting a subset, re-training the classifier, and ending the process.
We argue that such a pipeline cannot utilize the information provided by the pseudo-labeled instances sufficiently.
To verify this, we run experiments with selecting different number of instances, and take different iterations in Figure~\ref{fig:iter-manner}.
Results suggest that ICI obtains better accuracy with iterative selection manner.
}
For example, select $6$ images with two iterations (ICI(3)) is superior to select $8$ images in one iteration (ICI(8)).
\revise{To make a balance between computational cost, and performance,  our experiments select $5$ images per iteration.}

\begin{table}[h]
\centering
\begin{tabular*}{\columnwidth}{@{\extracolsep{\fill}}lccccc}
\toprule 
Acc (\%)&0-10&10-20&20-30&30-40&40-50\tabularnewline
\midrule
b/t&0/0&0/0&1/2&7/16&91/133\tabularnewline
\midrule
\midrule
Acc (\%)&50-60&60-70&70-80&80-90&90-100\tabularnewline
\midrule
b/t&312/446&526/663&464/544&154/191&3/5\tabularnewline
\bottomrule
\end{tabular*}
\caption{\label{tab:stat}
We run 2000 episodes, with each episode training an initial classifier.
We denote 
``Acc'' as the accuracy intervals; and 
``b/T'' as the number of classifiers experienced improvement v.s. 
total classifiers in this accuracy interval. }
\end{table}
\mypar{Robustness against initial classifier.} 
What are the requirements for the initial linear classifier? Is it necessary to satisfy that the accuracy of the initial linear classifier is higher than 50\% or even higher? 
The answer is no.
As long as the initial linear classifier can be trained, theoretically our method should work.  
\revise{It thus is a future open question of the influence of initial classifier.}
We briefly validate it in Table~\ref{tab:stat}.
We run 2000 episodes, with each episode training an initial classifier with different classification accuracy.
Table~\ref{tab:stat}
shows that most classifiers can get improved by ICI regardless of the initial accuracy.

\begin{table}[H]
\centering
\begin{tabular*}{\columnwidth}{@{\extracolsep{\fill}}lcccc}
\toprule 
\multirow{2}{*}{Model}&\multicolumn{2}{c}{Tran.}&\multicolumn{2}{c}{Semi.}\tabularnewline
&1shot&5shot&1shot&5shot\tabularnewline
\midrule
kNN & $71.45$\ci{0.61}&$79.88$\ci{0.38}& $69.14$\ci{0.57} & $77.20$\ci{0.38} \tabularnewline
 SVM & $72.13$\ci{0.62} & $82.76$\ci{0.34}& $70.76$\ci{0.58} & $80.83$\ci{0.35}\tabularnewline
 LR & $72.39$\ci{0.62}& $83.27$\ci{0.33}& $72.32$\ci{0.58}& $82.78$\ci{0.33}\tabularnewline
\bottomrule
\end{tabular*}
\caption{\label{tab:classifiers}
Performance of ICI using different classifiers on \emph{ mini}ImageNet under several settings.}
\end{table}

\mypar{Robustness against choices of classifiers.}
\revise{
Naturally, our proposed ICI is orthogonal to the choices of classifiers. 
To verify this, we select two other popular machine learning classifiers, linear support vector machine and k-nearest neighbor classifier, and run the SSFSL/TFSL 1-shot/5-shot tasks on the \emph{mini}ImageNet dataset. 
From results listed in Table~\ref{tab:classifiers}, the performance on 1-shot task is comparable, while on 5-shot task LR is superior to the other two classifiers.
Thus, one can select the classifier which fits best in their own task and still enjoy the improvements given by ICI.
}

\mypar{Influence of reduced dimension.}  
In this part, we study the influence of reduced dimension $d$ in our algorithm on $5$-way $1$-shot \textit{mini}ImageNet experiments.
The results with reduced dimension $2$, $5$, $10$, $20$, $50$, and without dimensionality reduction \ie, $d=512$, are shown in Table~\ref{tab:reduced}. 
Our algorithm achieves better performance when the reduced dimension is much smaller than the number of instances (\ie, $d\ll n$), which is consistent with the theoretical property~\cite{fan2018partial}. 
Moreover, we can observe that our model achieves the best accuracy of $72.39\%$ when $d=5$.
Practically, we adopt $d=5$ in our model.

\begin{table}[H]
\begin{centering}
\begin{tabular*}{\columnwidth}{@{\extracolsep{\fill}}lclc}
\toprule
$d$ & Acc (\%)&Alg.&Acc (\%)\tabularnewline
\cmidrule{1-2} \cmidrule{3-4}
$2$ & $70.03$\ci{0.58}&Isomap~\cite{tenenbaum2000global} & $71.49$\ci{0.60}\tabularnewline
$5$ & $\bf72.39$\ci{0.62}&PCA~\cite{tipping1999probabilistic} & $71.52$\ci{0.63}\tabularnewline
$10$ & $71.80$\ci{0.61}&LTSA~\cite{zhang2004principal} & $70.10$\ci{0.59}\tabularnewline
$20$ & $71.17$\ci{0.59}&MDS~\cite{borg2003modern} & $68.05$\ci{0.53}\tabularnewline
$50$ & $69.30 $\ci{0.55}&LLE~\cite{roweis2000nonlinear} & $72.39 $\ci{0.62}\tabularnewline
$512$ & $67.08$\ci{0.51}&SE~\cite{belkin2003laplacian} & $72.43$\ci{0.63} \tabularnewline 
\bottomrule
\end{tabular*}
\par\end{centering}
\caption{\label{tab:reduced}Influence of reduced dimension and dimension reduction  algorithms.}
\end{table}

\mypar{Influence of dimension reduction algorithms.} 
Furthermore, we study the robustness of ICI to different dimension reduction algorithms.
We compare
Isomap~\cite{tenenbaum2000global},
principal components analysis~\cite{tipping1999probabilistic} (PCA),
local tangent space alignment~\cite{zhang2004principal} (LTSA),
multi-dimensional scaling~\cite{borg2003modern} (MDS),
locally linear embedding~\cite{roweis2000nonlinear} (LLE) and
spectral embedding~\cite{belkin2003laplacian} (SE)
on $5$-way $1$-shot \textit{mini}ImageNet experiments.
From Table~\ref{tab:reduced} we can observe that the performance of 
ICI is comparable across most of the dimensionality reduction algorithms (from LTAS $70.10\%$ to SE $72.43\%$) except MDS ($68.05\%$).
We adopt LLE for dimension reduction in our method.

\begin{table}[H]
\begin{centering}
\begin{tabular*}{\columnwidth}{@{\extracolsep{\fill}}lllcc}
\toprule 
\multirow{2}{*}{Features} & \multirow{2}{*}{Backbone}& \multirow{2}{*}{Task} & \multicolumn{2}{c}{Accuracy}\tabularnewline
 &  && Competitors & ICIR\tabularnewline
\midrule 
\multirow{2}{*}{CAN~\cite{hou2019cross}} & \multirow{2}{*}{ResNet-12} &1-shot& $67.19$\ci{0.55} & $70.53$\ci{0.63} \tabularnewline
& & 5-shot&$80.64$\ci{0.35} & $81.30$\ci{0.36}\tabularnewline
\midrule 
\multirow{2}{*}{E$^{3}$BM~\cite{Liu2020E3BM}} & \multirow{2}{*}{WRN-28-10} &1-shot& $71.4$ & $71.39$\ci{0.63} \tabularnewline
& & 5-shot& $81.2$ & $82.61$\ci{0.36}\tabularnewline
\midrule 
\multirow{2}{*}{TAFSSL~\cite{lichtenstein2020tafssl}} & \multirow{2}{*}{DenseNet} &1-shot& $77.06$\ci{0.26} & $76.83$\ci{0.60} \tabularnewline
&  & 5-shot& $84.99$\ci{0.14}& $85.12$\ci{0.32} \tabularnewline
\bottomrule
\end{tabular*}
\end{centering}
\caption{\label{tab:backbone}
Comparison under different backbones with exactly the same features.}
\end{table}
\mypar{Influence of backbone.}
\revise{
One might wonder how does the backbone influences the performance of ICI.
In this part, we select three different competitors with different backbones, including ResNet-12, ResNet-18, and WideResNet.
We use their pre-trained model to ensure that we are using exactly the  same features in experiments.
The transudctive few-shot learning results is listed in Table~\ref{tab:backbone}, from where we could find that ICI enjoys comparable or even better performance with different backbones using only a simple linear classifier.
Hence the effectiveness of ICI does not depend on the selection of backbone.
}

\begin{figure}[h]
\centering
\includegraphics[width=1\columnwidth]{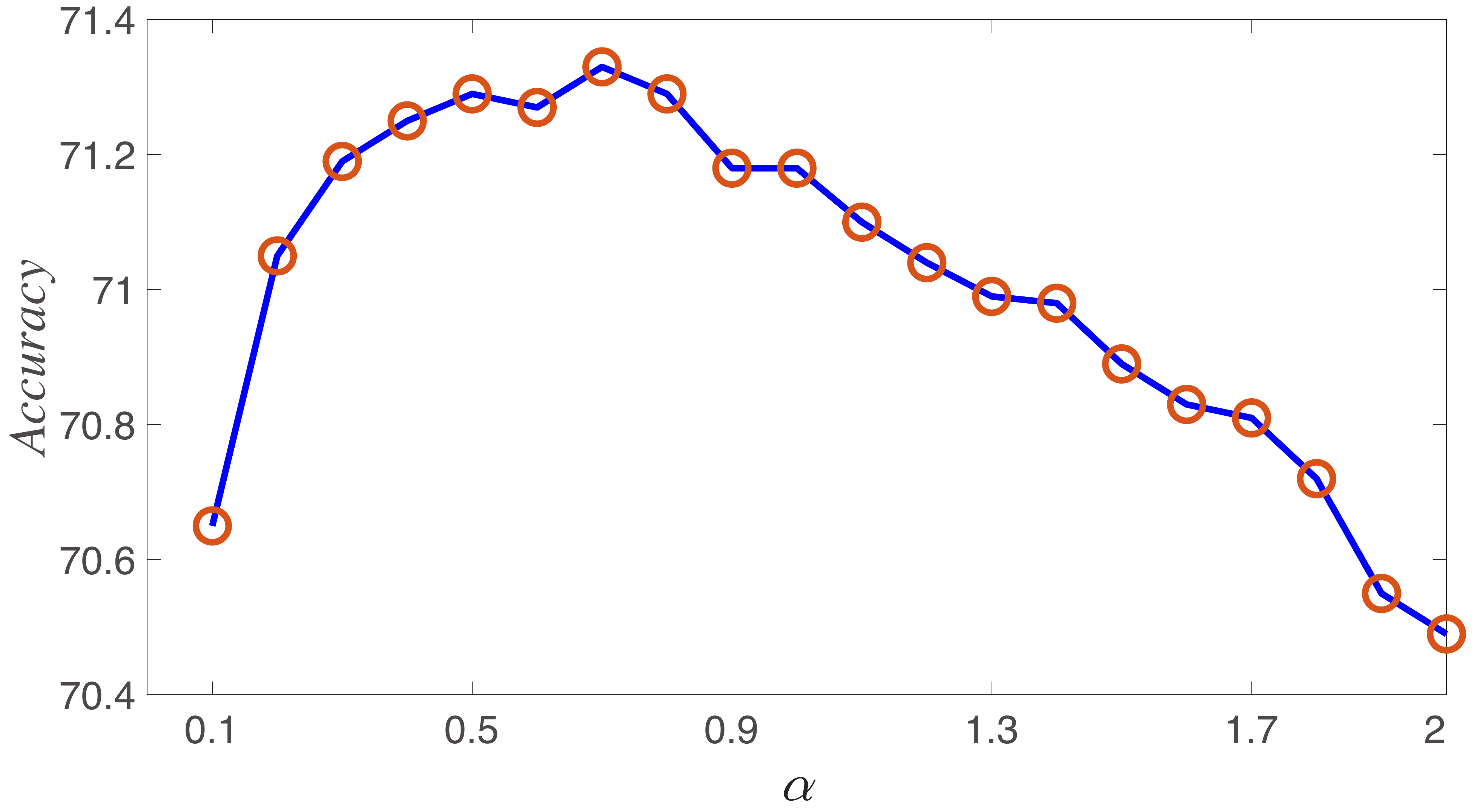} 
\caption{\label{fig:lr-penalty}
Validation accuracy with different $\alpha$s.
}
\end{figure}
\mypar{Influence of the penalty of logistic regression coefficient in ICI.}
In Section~\ref{sec:extension-lr}, 
we have shown that the penalty of the logistic regression coefficient is necessary for a unique solution. 
However, this introduces the hyper-parameters $\lambda_1$ and $\lambda_2$ which we need to trade-off. 
Note that since we still aim to find the solution path of $\bg$, which is solved when we use a list of $\lambda_2$s.
We set $\lambda_1 = \alpha\lambda_2$ for each solution point along the path and search for the best $\alpha$ based on the inference performance on the validation set.
Results are shown in Fig.~\ref{fig:lr-penalty}, indicating that the performance is maximized when $\alpha$ is set around $0.5$.
In our experiments, we use $\alpha=0.5$.

\section{Conclusion}

In this paper, we have proposed a statistical method, called Instance Credibility Inference (ICI) to exploit the distribution support of unlabeled instances for few-shot \revise{visual recognition}. 
The proposed ICI effectively select the most trustworthy pseudo-labeled instances according to their credibility to augment the training set. 
In order to measure the credibility of each pseudo-labeled instance, 
we propose to solve a hypothesis by increasing the sparsity of the incidental parameters and rank the pseudo-labeled instance \revise{according to} their sparsity degree. 
Theoretical analysis shows that under conditions of \textit{restricted eigenvalue, irrepresentability, and large error}, our ICI \revise{is able to find} all the correctly-predicted instances from the noisy pseudo-labeled set.
Extensive experiments show that our simple approach achieves appealing performance on four widely used few-shot \revise{visual recognition} benchmark datasets including \textit{mini}ImageNet, \textit{tiered}ImageNet, CIFAR-FS, and CUB.

\bibliographystyle{IEEEtran}
\bibliography{ici.bib}

\begin{thebibliography}{10}
\providecommand{\url}[1]{#1}
\csname url@samestyle\endcsname
\providecommand{\newblock}{\relax}
\providecommand{\bibinfo}[2]{#2}
\providecommand{\BIBentrySTDinterwordspacing}{\spaceskip=0pt\relax}
\providecommand{\BIBentryALTinterwordstretchfactor}{4}
\providecommand{\BIBentryALTinterwordspacing}{\spaceskip=\fontdimen2\font plus
\BIBentryALTinterwordstretchfactor\fontdimen3\font minus
  \fontdimen4\font\relax}
\providecommand{\BIBforeignlanguage}[2]{{%
\expandafter\ifx\csname l@#1\endcsname\relax
\typeout{** WARNING: IEEEtran.bst: No hyphenation pattern has been}%
\typeout{** loaded for the language `#1'. Using the pattern for}%
\typeout{** the default language instead.}%
\else
\language=\csname l@#1\endcsname
\fi
#2}}
\providecommand{\BIBdecl}{\relax}
\BIBdecl

\bibitem{wang2020generalizing}
Y.~Wang, Q.~Yao, J.~T. Kwok, and L.~M. Ni, ``Generalizing from a few examples:
  A survey on few-shot learning,'' \emph{ACM Computing Surveys (CSUR)}, 2020.

\bibitem{zhang2017learning}
L.~Zhang, T.~Xiang, and S.~Gong, ``Learning a deep embedding model for
  zero-shot learning,'' in \emph{IEEE Conference on Computer Vision and Pattern
  Recognition}, 2017.

\bibitem{krizhevsky2012imagenet}
A.~Krizhevsky, I.~Sutskever, and G.~E. Hinton, ``Imagenet classification with
  deep convolutional neural networks,'' in \emph{Advances in Neural Information
  Processing Systems}, 2012.

\bibitem{simonyan2014very}
K.~Simonyan and A.~Zisserman, ``Very deep convolutional networks for
  large-scale image recognition,'' \emph{International Conference on Learning
  Representations}, 2015.

\bibitem{he2016deep}
K.~He, X.~Zhang, S.~Ren, and J.~Sun, ``Deep residual learning for image
  recognition,'' in \emph{IEEE Conference on Computer Vision and Pattern
  Recognition}, 2016.

\bibitem{huang2017densely}
G.~Huang, Z.~Liu, L.~Van Der~Maaten, and K.~Q. Weinberger, ``Densely connected
  convolutional networks,'' in \emph{IEEE Conference on Computer Vision and
  Pattern Recognition}, 2017.

\bibitem{finn2017model}
C.~Finn, P.~Abbeel, and S.~Levine, ``Model-agnostic meta-learning for fast
  adaptation of deep networks,'' in \emph{International Conference on Machine
  Learning}, 2017.

\bibitem{snell2017prototypical}
J.~Snell, K.~Swersky, and R.~Zemel, ``Prototypical networks for few-shot
  learning,'' in \emph{Advances in Neural Information Processing Systems},
  2017.

\bibitem{sung2018learning}
F.~Sung, Y.~Yang, L.~Zhang, T.~Xiang, P.~H. Torr, and T.~M. Hospedales,
  ``Learning to compare: Relation network for few-shot learning,'' in
  \emph{IEEE Conference on Computer Vision and Pattern Recognition}, 2018.

\bibitem{vinyals2016matching}
O.~Vinyals, C.~Blundell, T.~Lillicrap, D.~Wierstra \emph{et~al.}, ``Matching
  networks for one shot learning,'' in \emph{Advances in Neural Information
  Processing Systems}, 2016.

\bibitem{yosinski2014transferable}
J.~Yosinski, J.~Clune, Y.~Bengio, and H.~Lipson, ``How transferable are
  features in deep neural networks?'' in \emph{Advances in Neural Information
  Processing Systems}, 2014.

\bibitem{chen2019image}
Z.~Chen, Y.~Fu, Y.-X. Wang, L.~Ma, W.~Liu, and M.~Hebert, ``Image deformation
  meta-networks for one-shot learning,'' in \emph{IEEE Conference on Computer
  Vision and Pattern Recognition}, 2019.

\bibitem{chen2019multi}
Z.~Chen, Y.~Fu, Y.~Zhang, Y.-G. Jiang, X.~Xue, and L.~Sigal, ``Multi-level
  semantic feature augmentation for one-shot learning,'' \emph{IEEE
  Transactions on Image Processing}, 2019.

\bibitem{lemke2015metalearning}
C.~Lemke, M.~Budka, and B.~Gabrys, ``Metalearning: a survey of trends and
  technologies,'' \emph{Artificial intelligence review}, 2015.

\bibitem{oreshkin2018tadam}
B.~Oreshkin, P.~R. L{\'o}pez, and A.~Lacoste, ``Tadam: Task dependent adaptive
  metric for improved few-shot learning,'' in \emph{Advances in Neural
  Information Processing Systems}, 2018.

\bibitem{sung2017learning}
F.~Sung, L.~Zhang, T.~Xiang, T.~Hospedales, and Y.~Yang, ``Learning to learn:
  Meta-critic networks for sample efficient learning,'' \emph{arXiv preprint
  arXiv:1706.09529}, 2017.

\bibitem{li2017meta}
Z.~Li, F.~Zhou, F.~Chen, and H.~Li, ``Meta-sgd: Learning to learn quickly for
  few-shot learning,'' \emph{arXiv preprint arXiv:1707.09835}, 2017.

\bibitem{nichol2018first}
A.~Nichol, J.~Achiam, and J.~Schulman, ``On first-order meta-learning
  algorithms,'' \emph{arXiv preprint arXiv:1803.02999}, 2018.

\bibitem{rusu2018meta}
A.~A. Rusu, D.~Rao, J.~Sygnowski, O.~Vinyals, R.~Pascanu, S.~Osindero, and
  R.~Hadsell, ``Meta-learning with latent embedding optimization,'' in
  \emph{International Conference on Learning Representations}, 2019.

\bibitem{DBLP:journals/corr/abs-1904-04232}
W.~Chen, Y.~Liu, Z.~Kira, Y.~F. Wang, and J.~Huang, ``A closer look at few-shot
  classification,'' in \emph{International Conference on Learning
  Representations}, 2019.

\bibitem{Liu_2020_CVPR_Workshops}
C.~Liu, C.~Xu, Y.~Wang, L.~Zhang, and Y.~Fu, ``An embarrassingly simple
  baseline to one-shot learning,'' in \emph{IEEE Conference on Computer Vision
  and Pattern Recognition}, 2020.

\bibitem{liu2018learning}
Y.~Liu, J.~Lee, M.~Park, S.~Kim, E.~Yang, S.~Hwang, and Y.~Yang, ``Learning to
  propagate labels: Transductive propagation network for few-shot learning,''
  in \emph{International Conference on Learning Representations}, 2019.

\bibitem{ren2018meta}
M.~Ren, S.~Ravi, E.~Triantafillou, J.~Snell, K.~Swersky, J.~B. Tenenbaum,
  H.~Larochelle, and R.~S. Zemel, ``Meta-learning for semi-supervised few-shot
  classification,'' in \emph{International Conference on Learning
  Representations}, 2018.

\bibitem{sun2019learning}
X.~Li, Q.~Sun, Y.~Liu, Q.~Zhou, S.~Zheng, T.-S. Chua, and B.~Schiele,
  ``Learning to self-train for semi-supervised few-shot classification,'' 2019.

\bibitem{joachims1999transductive}
T.~Joachims, ``Transductive inference for text classification using support
  vector machines,'' in \emph{International Conference on Machine Learning},
  1999.

\bibitem{qiao2019transductive}
L.~Qiao, Y.~Shi, J.~Li, Y.~Wang, T.~Huang, and Y.~Tian, ``Transductive
  episodic-wise adaptive metric for few-shot learning,'' in \emph{IEEE
  International Conference on Computer Vision}, 2019.

\bibitem{self-taught-learning}
R.~Raina, A.~Battle, H.~Lee, B.~Packer, and A.~Y. Ng, ``Self-taught learning:
  Transfer learning from unlabeled data,'' in \emph{International Conference on
  Machine Learning}, 2007.

\bibitem{fan2018partial}
J.~Fan, R.~Tang, and X.~Shi, ``Partial consistency with sparse incidental
  parameters,'' \emph{Statistica Sinica}, vol.~28, p. 2633, 2018.

\bibitem{wang2020instance}
Y.~Wang, C.~Xu, C.~Liu, L.~Zhang, and Y.~Fu, ``Instance credibility inference
  for few-shot learning,'' in \emph{IEEE Conference on Computer Vision and
  Pattern Recognition}, 2020.

\bibitem{vapnik1998statistical}
V.~Vapnik and V.~Vapnik, ``Statistical learning theory wiley,'' \emph{New
  York}, 1998.

\bibitem{bennett1999semi}
K.~P. Bennett and A.~Demiriz, ``Semi-supervised support vector machines,'' in
  \emph{Advances in Neural Information Processing Systems}, 1999.

\bibitem{li2014towards}
Y.-F. Li and Z.-H. Zhou, ``Towards making unlabeled data never hurt,''
  \emph{IEEE Transactions on Pattern Analysis and Machine Intelligence}, 2014.

\bibitem{conf/iclr/LaineA17}
S.~Laine and T.~Aila, ``Temporal ensembling for semi-supervised learning.'' in
  \emph{International Conference on Learning Representations}, 2017.

\bibitem{tarvainen2017mean}
A.~Tarvainen and H.~Valpola, ``Mean teachers are better role models:
  Weight-averaged consistency targets improve semi-supervised deep learning
  results,'' in \emph{Advances in Neural Information Processing Systems}, 2017.

\bibitem{miayto2016virtual}
T.~Miayto, A.~M. Dai, and I.~Goodfellow, ``Virtual adversarial training for
  semi-supervised text classification,'' 2016.

\bibitem{NoisyStudent}
Q.~Xie, M.-T. Luong, E.~Hovy, and Q.~V. Le, ``Self-training with noisy student
  improves imagenet classification,'' 2020.

\bibitem{amini2002semi}
M.-R. Amini and P.~Gallinari, ``Semi-supervised logistic regression,'' in
  \emph{ECAI}, 2002.

\bibitem{grandvalet2005semi}
Y.~Grandvalet and Y.~Bengio, ``Semi-supervised learning by entropy
  minimization,'' in \emph{Advances in Neural Information Processing Systems},
  2005.

\bibitem{lee2013pseudo}
D.-H. Lee, ``Pseudo-label: The simple and efficient semi-supervised learning
  method for deep neural networks,'' in \emph{International Conference on
  Machine Learning workshops}, 2013.

\bibitem{arazo2019pseudo}
E.~Arazo, D.~Ortego, P.~Albert, N.~E. O'Connor, and K.~McGuinness,
  ``Pseudo-labeling and confirmation bias in deep semi-supervised learning,''
  2020.

\bibitem{iscen2019label}
A.~Iscen, G.~Tolias, Y.~Avrithis, and O.~Chum, ``Label propagation for deep
  semi-supervised learning,'' in \emph{IEEE Conference on Computer Vision and
  Pattern Recognition}, 2019.

\bibitem{shi2018transductive}
W.~Shi, Y.~Gong, C.~Ding, Z.~MaXiaoyu~Tao, and N.~Zheng, ``Transductive
  semi-supervised deep learning using min-max features,'' in \emph{European
  Conference on Computer Vision}, 2018.

\bibitem{angluin1988learning}
D.~Angluin and P.~Laird, ``Learning from noisy examples,'' \emph{Machine
  Learning}, vol.~2, no.~4, pp. 343--370, 1988.

\bibitem{song2020learning}
H.~Song, M.~Kim, D.~Park, and J.-G. Lee, ``Learning from noisy labels with deep
  neural networks: A survey,'' \emph{arXiv preprint arXiv:2007.08199}, 2020.

\bibitem{ghosh2017robust}
A.~Ghosh, H.~Kumar, and P.~Sastry, ``Robust loss functions under label noise
  for deep neural networks,'' in \emph{Proceedings of the AAAI Conference on
  Artificial Intelligence}, 2017.

\bibitem{goldberger2016training}
J.~Goldberger and E.~Ben-Reuven, ``Training deep neural-networks using a noise
  adaptation layer,'' 2016.

\bibitem{jenni2018deep}
S.~Jenni and P.~Favaro, ``Deep bilevel learning,'' in \emph{Proceedings of the
  European conference on computer vision (ECCV)}, 2018.

\bibitem{chang2017active}
H.-S. Chang, E.~Learned-Miller, and A.~McCallum, ``Active bias: Training more
  accurate neural networks by emphasizing high variance samples,'' 2017.

\bibitem{ICML2019_UnsupervisedLabelNoise}
E.~Arazo, D.~Ortego, P.~Albert, N.~E. O'Connor, and K.~McGuinness,
  ``Unsupervised label noise modeling and loss correction,'' in \emph{ICML},
  2019.

\bibitem{song2020robust}
J.~Song, Y.~Dauphin, M.~Auli, and T.~Ma, ``Robust and on-the-fly dataset
  denoising for image classification,'' in \emph{European Conference on
  Computer Vision}, 2020.

\bibitem{huang2019o2u}
J.~Huang, L.~Qu, R.~Jia, and B.~Zhao, ``O2u-net: A simple noisy label detection
  approach for deep neural networks,'' in \emph{Proceedings of the IEEE/CVF
  International Conference on Computer Vision}, 2019, pp. 3326--3334.

\bibitem{mishra2018a}
N.~Mishra, M.~Rohaninejad, X.~Chen, and P.~Abbeel, ``A simple neural attentive
  meta-learner,'' in \emph{International Conference on Learning
  Representations}, 2018.

\bibitem{hou2019cross}
R.~Hou, H.~Chang, B.~Ma, S.~Shan, and X.~Chen, ``Cross attention network for
  few-shot classification,'' \emph{NeurIPS}, 2019.

\bibitem{hu2020exploiting}
Y.~Hu, V.~Gripon, and S.~Pateux, ``Exploiting unsupervised inputs for accurate
  few-shot classification,'' \emph{arXiv preprint}, 2020.

\bibitem{lichtenstein2020tafssl}
M.~Lichtenstein, P.~Sattigeri, R.~Feris, R.~Giryes, and L.~Karlinsky, ``Tafssl:
  Task-adaptive feature sub-space learning for few-shot classification,''
  \emph{ECCV}, 2020.

\bibitem{yang2020dpgn}
L.~Yang, L.~Li, Z.~Zhang, X.~Zhou, E.~Zhou, and Y.~Liu, ``Dpgn: Distribution
  propagation graph network for few-shot learning,'' in \emph{CVPR}, 2020.

\bibitem{hu2020leveraging}
Y.~Hu, V.~Gripon, and S.~Pateux, ``Leveraging the feature distribution in
  transfer-based few-shot learning,'' \emph{arXiv preprint}, 2020.

\bibitem{kye2020transductive}
S.~M. Kye, H.~B. Lee, H.~Kim, and S.~J. Hwang, ``Transductive few-shot learning
  with meta-learned confidence,'' \emph{arXiv preprint}, 2020.

\bibitem{neyman1948consistent}
J.~Neyman and E.~L. Scott, ``Consistent estimates based on partially consistent
  observations,'' \emph{Econometrica: Journal of the Econometric Society},
  1948.

\bibitem{fan2010selective}
J.~Fan and J.~Lv, ``A selective overview of variable selection in high
  dimensional feature space,'' \emph{Statistica Sinica}, 2010.

\bibitem{kiefer1956consistency}
J.~Kiefer and J.~Wolfowitz, ``Consistency of the maximum likelihood estimator
  in the presence of infinitely many incidental parameters,'' \emph{The Annals
  of Mathematical Statistics}, 1956.

\bibitem{basu2011elimination}
D.~Basu, ``On the elimination of nuisance parameters,'' in \emph{Selected Works
  of Debabrata Basu}, 2011.

\bibitem{moreira2008maximum}
M.~Moreira, ``A maximum likelihood method for the incidental parameter
  problem,'' Tech. Rep., 2008.

\bibitem{fu2015robust}
Y.~Fu, T.~M. Hospedales, T.~Xiang, J.~Xiong, S.~Gong, Y.~Wang, and Y.~Yao,
  ``Robust subjective visual property prediction from crowdsourced pairwise
  labels,'' \emph{IEEE Transactions on Pattern Analysis and Machine
  Intelligence}, 2015.

\bibitem{roweis2000nonlinear}
S.~T. Roweis and L.~K. Saul, ``Nonlinear dimensionality reduction by locally
  linear embedding,'' \emph{science}, 2000.

\bibitem{simon2013blockwise}
N.~Simon, J.~Friedman, and T.~Hastie, ``A blockwise descent algorithm for
  group-penalized multiresponse and multinomial regression,'' \emph{arXiv
  preprint arXiv:1311.6529}, 2013.

\bibitem{zhu1997algorithm}
C.~Zhu, R.~H. Byrd, P.~Lu, and J.~Nocedal, ``Algorithm 778: L-bfgs-b: Fortran
  subroutines for large-scale bound-constrained optimization,'' \emph{ACM
  Transactions on Mathematical Software}, 1997.

\bibitem{fan2008liblinear}
R.-E. Fan, K.-W. Chang, C.-J. Hsieh, X.-R. Wang, and C.-J. Lin, ``Liblinear: A
  library for large linear classification,'' \emph{Journal of Machine Learning
  Research}, 2008.

\bibitem{yu2011dual}
H.-F. Yu, F.-L. Huang, and C.-J. Lin, ``Dual coordinate descent methods for
  logistic regression and maximum entropy models,'' \emph{Machine Learning},
  2011.

\bibitem{tikhonov1977solutions}
A.~N. Tikhonov and V.~Y. Arsenin, ``Solutions of ill-posed problems,''
  \emph{New York}, 1977.

\bibitem{zhao2006model}
P.~Zhao and B.~Yu, ``On model selection consistency of lasso,'' \emph{Journal
  of Machine learning research}, vol.~7, no. Nov, pp. 2541--2563, 2006.

\bibitem{wainwright2009sharp}
M.~J. Wainwright, ``Sharp thresholds for high-dimensional and noisy sparsity
  recovery using $\ell_ {1}$ -constrained quadratic programming (lasso),''
  \emph{IEEE transactions on information theory}, 2009.

\bibitem{ravi2016optimization}
S.~Ravi and H.~Larochelle, ``Optimization as a model for few-shot learning,''
  in \emph{International Conference on Learning Representations}, 2017.

\bibitem{bertinetto2018metalearning}
L.~Bertinetto, J.~F. Henriques, P.~Torr, and A.~Vedaldi, ``Meta-learning with
  differentiable closed-form solvers,'' in \emph{International Conference on
  Learning Representations}, 2019.

\bibitem{wah2011caltech}
C.~Wah, S.~Branson, P.~Welinder, P.~Perona, and S.~Belongie, ``{The
  Caltech-UCSD Birds-200-2011 Dataset},'' California Institute of Technology,
  Tech. Rep., 2011.

\bibitem{hilliard2018few}
N.~Hilliard, L.~Phillips, S.~Howland, A.~Yankov, C.~D. Corley, and N.~O. Hodas,
  ``Few-shot learning with metric-agnostic conditional embeddings,''
  \emph{arXiv preprint arXiv:1802.04376}, 2018.

\bibitem{triantafillou2017few}
E.~Triantafillou, R.~Zemel, and R.~Urtasun, ``Few-shot learning through an
  information retrieval lens,'' in \emph{Advances in Neural Information
  Processing Systems}, 2017.

\bibitem{krizhevsky2009learning}
A.~Krizhevsky, G.~Hinton \emph{et~al.}, ``Learning multiple layers of features
  from tiny images,'' Citeseer, Tech. Rep., 2009.

\bibitem{ye2020fewshot}
H.-J. Ye, H.~Hu, D.-C. Zhan, and F.~Sha, ``Few-shot learning via embedding
  adaptation with set-to-set functions,'' in \emph{IEEE/CVF Conference on
  Computer Vision and Pattern Recognition (CVPR)}, 2020, pp. 8808--8817.

\bibitem{Liu2020E3BM}
Y.~Liu, B.~Schiele, and Q.~Sun, ``An ensemble of epoch-wise empirical bayes for
  few-shot learning,'' in \emph{European Conference on Computer Vision (ECCV)},
  2020.

\bibitem{Zhang_2020_CVPR}
C.~Zhang, Y.~Cai, G.~Lin, and C.~Shen, ``Deepemd: Few-shot image classification
  with differentiable earth mover's distance and structured classifiers,'' in
  \emph{IEEE/CVF Conference on Computer Vision and Pattern Recognition (CVPR)},
  June 2020.

\bibitem{lee2019meta}
K.~Lee, S.~Maji, A.~Ravichandran, and S.~Soatto, ``Meta-learning with
  differentiable convex optimization,'' in \emph{IEEE Conference on Computer
  Vision and Pattern Recognition}, 2019.

\bibitem{DBLP:journals/corr/HeZRS15}
K.~He, X.~Zhang, S.~Ren, and J.~Sun, ``Deep residual learning for image
  recognition,'' \emph{CoRR}, 2015.

\bibitem{JMLR:v15:srivastava14a}
N.~Srivastava, G.~Hinton, A.~Krizhevsky, I.~Sutskever, and R.~Salakhutdinov,
  ``Dropout: A simple way to prevent neural networks from overfitting,''
  \emph{Journal of Machine Learning Research}, 2014.

\bibitem{ghiasi2018dropblock}
G.~Ghiasi, T.-Y. Lin, and Q.~V. Le, ``Dropblock: A regularization method for
  convolutional networks,'' in \emph{Advances in Neural Information Processing
  Systems}, 2018.

\bibitem{CAN}
R.~Hou, H.~Chang, B.~Ma, S.~Shan, and X.~Chen, ``Cross attention network for
  few-shot classification,'' in \emph{NeurIPS}, 2019.

\bibitem{munkhdalai2018rapid}
T.~Munkhdalai, X.~Yuan, S.~Mehri, and A.~Trischler, ``Rapid adaptation with
  conditionally shifted neurons,'' in \emph{International Conference on Machine
  Learning}, 2018.

\bibitem{yoon2019tapnet}
S.~W. Yoon, J.~Seo, and J.~Moon, ``Tapnet: Neural network augmented with
  task-adaptive projection for few-shot learning,'' in \emph{International
  Conference on Machine Learning}.\hskip 1em plus 0.5em minus 0.4em\relax PMLR,
  2019, pp. 7115--7123.

\bibitem{li2019finding}
H.~Li, D.~Eigen, S.~Dodge, M.~Zeiler, and X.~Wang, ``Finding task-relevant
  features for few-shot learning by category traversal,'' in \emph{IEEE
  Conference on Computer Vision and Pattern Recognition}, 2019.

\bibitem{tenenbaum2000global}
J.~B. Tenenbaum, V.~De~Silva, and J.~C. Langford, ``A global geometric
  framework for nonlinear dimensionality reduction,'' \emph{science}, 2000.

\bibitem{tipping1999probabilistic}
M.~E. Tipping and C.~M. Bishop, ``Probabilistic principal component analysis,''
  \emph{Journal of the Royal Statistical Society: Series B (Statistical
  Methodology)}, 1999.

\bibitem{zhang2004principal}
Z.~Zhang and H.~Zha, ``Principal manifolds and nonlinear dimensionality
  reduction via tangent space alignment,'' \emph{SIAM journal on scientific
  computing}, 2004.

\bibitem{borg2003modern}
I.~Borg and P.~Groenen, ``Modern multidimensional scaling: Theory and
  applications,'' \emph{Journal of Educational Measurement}, 2003.

\bibitem{belkin2003laplacian}
M.~Belkin and P.~Niyogi, ``Laplacian eigenmaps for dimensionality reduction and
  data representation,'' \emph{Neural computation}, 2003.

\bibitem{xu2021evaluating}
Q.~Xu, J.~Xiong, X.~Cao, Q.~Huang, and Y.~Yao, ``Evaluating visual properties
  via robust hodgerank,'' \emph{International Journal of Computer Vision}, pp.
  1--22, 2021.

\end{thebibliography}

\appendix[Proof of Theorem~\ref{thm:sufficiency}.]
\label{appendix}
\begin{prop}
\label{prop}
Assume that $\tilde{\bU}^\top\tilde{\bU}$ is invertible. 
If 
\begin{equation}
\label{appendix:prop-condition}
\left\Vert\lambda  \tilde{\bU}_{S^{c}}^{\top} \tilde{\bU}_{S}\left(\tilde{\bU}_{S}^{\top} \tilde{\bU}_{S}\right)^{-1}  \hat{\bm{v}}_{S}+ \tilde{\bU}_{S^{c}}^{\top}\left(\bm{I}-\bm{I}_S\right)(\tilde{\bU} \be)\right\Vert_{\infty}<\lambda
\end{equation}
holds for all $\hat{\bm{v}}_{S} \in[-1,1]^{S}$, where $\bm{I}_{S}=\tilde{\bU}_{S}\left(\tilde{\bU}_{S}^{\top} \tilde{\bU}_{S}\right)^{-1} \tilde{\bU}_{S}^{\top}$,
then the estimator $\hat{\vec{\bg}}$ of Eq.~\eqref{eq:thm-problem} satisfies that
\begin{equation*}
\hat{S}=\mathrm{supp}\left(\hat{\vec{\bg}}\right)\subseteq \mathrm{supp}\left(\vec{\bg}^*\right)=S. 
\end{equation*}

\noindent Moreover,
if the sign consistency
\begin{equation}
\label{appendix:sign-consisty}
\operatorname{sign}\left(\hat{\vec{\bg}}_S\right)=\operatorname{sign}\left(\vec{\bg}^{*}_S\right)
\end{equation}
holds,
Then 
$\hat{\vec{\bg}}$ is the unique solution of~\eqref{eq:thm-problem} with the same sign as $\hat{\vec{\bg}}^*$.
\end{prop}

\begin{proof}
\minor{
Note that Eq.~\eqref{eq:thm-problem} is convex that has global minima. 
Denote Eq.~\eqref{eq:thm-problem} as $L$, the solution of $\partial L/\partial\vec{\bm{\gamma}}=0$ is the unique minimizer. 
Hence we have
\begin{equation}
\frac{\partial L}{\partial\vec{\bm{\gamma}}}=-\tilde{\bm{U}}^{\top}\left(\vec{\bm{y}}_u-\tilde{\bm{U}}\vec{\bg}\right)+\lambda\bm{v}=0
\end{equation}
where $\bm{v}=\partial \left\Vert\vec{\bg}\right\Vert_1 /\partial\vec{\bm{\gamma}}$.
Note that $\left\Vert\vec{\bg}\right\Vert_1 $ is non-differentiable, so we
instead compute its sub-gradient. 
Further note that $v_{i}=\partial\left\Vert \vec{\bm{\gamma}}\right\Vert _{1}/\partial\vec{\gamma}_{i}=\partial\left|\vec{\gamma}_{i}\right|/\partial\gamma_{i}$. 
Hence $v_{i}=\mathrm{sign}\left(\vec{\gamma}_{i}\right)$ if $\vec{\gamma}_{i}\neq0$ and $v_{i}\in\left[-1,1\right]$ if $\vec{\gamma}_{i}=0$.
To distinguish between the two cases, we assume $v_{i}\in\left(-1,1\right)$ if $\vec{\gamma}_{i} = 0$.
Hence there exists $\hat{\bm{v}}\in\mathbb{R}^{n\times1}$ such that
}
\begin{equation}
-\tilde{\bU}^{\top}\left(\vec{\bm{y}}_u-\tilde{\bU}  \hat{\vec{\bg}}\right)+\lambda  \hat{\bm{v}}=0,
\end{equation}
where \minor{$\hat{v}_{i}=\mathrm{sign}\left(\hat{\vec{\gamma}}_{i}\right)$ if $i\in\hat{S}$ and $\hat{v}_{i}\in(-1,1)$ if $i\in\hat{S}^c$}.

To obtain $\hat{S} \subseteq S$, 
\minor{we should have $\hat{\vec{\gamma}}_{i}=0$ for $i \in S^{c}$, that is, $\forall i\in S^{c},\left|\hat{v}_{i}\right|<1$, \ie}
\begin{equation}
\left\Vert\tilde{\bU}_{S^{c}}^{\top}\left(\vec{\bm{y}}_u-\tilde{\bU}_{S}  \hat{\vec{\bg}}_S\right)\right\Vert_{\infty}<\lambda,
\label{eq:prop-inequality}
\end{equation}
For \minor{$i \in S$}, we have
\begin{equation}
-\tilde{\bU}_{S}^{\top}\left(\vec{\bm{y}}_u-\tilde{\bU}_{S} \hat{\vec{\bg}}_{S}\right)+\lambda \hat{\bm{v}}_{S}=0.
\end{equation}
If $\tilde{\bU}^\top\tilde{\bU}$ is invertible then
\minor{
\begin{equation}
\hat{\vec{\bm{\gamma}}}_{S}=\left(\tilde{\bm{U}}_{S}^{\top}\tilde{\bm{U}}_{S}\right)^{-1}\left(\tilde{\bm{U}}_{S}^{\top}\vec{\bm{y}}_u-\lambda\hat{\bm{v}}_{S}\right)
\end{equation}
Recall that 
we have
\begin{equation}
\vec{\bm{y}}_u=\tilde{\bm{U}}_{S}\vec{\bm{\gamma}}_{S}^{*}+\tilde{\bm{U}}\vec{\bm{\varepsilon}}
\label{eq:y-ground-truth-appendix}
\end{equation}
Hence
}
\begin{equation}
\hat{\vec{\bg}}_{S}=\vec{\bg}_{S}^{*}+\delta_{S}, \quad \delta_{S}:=\left(\tilde{\bU}_{S}^{\top} \tilde{\bU}_{S}\right)^{-1}\left[\tilde{\bU}_{S}^{\top} \tilde{\bU} \vec{\bm{\varepsilon}}-\lambda\hat{\bm{v}}_{S}\right].
\label{eq:prop-condition}
\end{equation}
Plugging~\eqref{eq:prop-condition} and~\eqref{eq:y-ground-truth-appendix} into~\eqref{eq:prop-inequality} we have
\begin{equation}
\left\Vert\tilde{\bU}_{S^{c}}^{\top} \tilde{\bU} \vec{\bm{\varepsilon}}- \tilde{\bU}_{S^{c}}^{\top} \tilde{\bU}_{S}\left(\tilde{\bU}_{S}^{\top} \tilde{\bU}_{S}\right)^{-1}\left[\tilde{\bU}_{S}^{\top} \tilde{\bU} \vec{\bm{\varepsilon}}-\lambda \hat{\bm{v}}_{S}\right]\right\Vert_{\infty}<\lambda,
\end{equation}
or equivalently
\begin{equation}
\left\Vert\lambda  \tilde{\bU}_{S^{c}}^{\top} \tilde{\bU}_{S}\left(\tilde{\bU}_{S}^{\top} \tilde{\bU}_{S}\right)^{-1}  \hat{\bm{v}}_{S}+ \tilde{\bU}_{S^{c}}^{\top}\left(\bm{I}-\bm{I}_S\right)\tilde{\bU} \vec{\bm{\varepsilon}}\right\Vert_{\infty}<\lambda,
\end{equation}
where $\bm{I}_S=\tilde{\bU}_{S}\left(\tilde{\bU}_{S}^{\top} \tilde{\bU}_{S}\right)^{-1} \tilde{\bU}_{S}^{\top}$.
To ensure the sign consistency, replacing $\hat{\bm{v}}_{S}=\operatorname{sign}\left(\vec{\bg}_{S}^{*}\right)$ in the inequality above leads to the final result.
\end{proof}
\begin{lem}
\label{lemma}
Assume that $\vec{\bm{\varepsilon}}$ is 
\minor{indenpendent sub-Gaussian with zero mean and bounded variance
$\mathrm{Var}\left(\vec{\bm{\varepsilon}}_{i}\right)\leq\sigma^2$.
}
Then with probability at least
\begin{equation}
1-2 c n \exp \left(-\frac{\lambda^{2} \eta^{2} }{2  \sigma^{2} \max_{i\in S^{c}}\left\Vert \tilde{\bm{U}}_{i}\right\Vert _{2}^{2}}\right)
\end{equation}
there holds
\begin{equation}
\label{appendix:first-bound}
\left\|\tilde{\bU}_{S^{c}}^{\top}\left(\bm{I}-\bm{I}_{S}\right)\left(\tilde{\bU} \vec{\be}\right)\right\|_{\infty} \leq \lambda \eta
\end{equation}
and
\begin{equation}
\label{appendix:second-bound}
\left\|\left(\tilde{\bU}_{S}^{\top} \tilde{\bU}_{S}\right)^{-1} \tilde{\bU}_{S}^{\top} \tilde{\bU} \vec{\be}\right\|_{\infty} \leq \frac{\lambda \eta}{ \sqrt{C_{\min }} \max_{i\in S^{c}}\left\Vert \tilde{\bm{U}}_{i}\right\Vert _{2}}.
\end{equation}
\end{lem}
\begin{proof}
Let $\bm{z}^c= \tilde{\bU}_{S^{c}}^{\top}\left(\bm{I}-\bm{I}_{S}\right)\left(\tilde{\bU} \vec{\be}\right)$, for each  \minor{$i\in S^c$} the variance can be bounded by
\begin{equation*}
\operatorname{Var}\left(\bm{z}_i^c\right) \leq\sigma^{2}\tilde{\bm{U}}_{i}^{\top}\left(\bm{I}-\bm{I}_{S}\right)^{2}\tilde{\bm{U}}_{i} \leq  \sigma^{2} \max_{i\in S^{c}}\left\Vert \tilde{\bm{U}}_{i}\right\Vert _{2}^{2}.
\end{equation*}
Hoeffding inequality implies that 
\begin{equation*}
\begin{aligned}
& \mathbb{P}\left(\left\| \tilde{\bU}_{S^{c}}^{\top}\left(\bm{I}-\bm{I}_{S}\right)\left(\tilde{\bU} \vec{\be}\right)\right\|_{\infty} \geq t\right)\\
& \leq 2\left|S^{c}\right| \exp \left(-\frac{t^{2} }{2  \sigma^{2} \max_{i\in S^{c}}\left\Vert \tilde{\bm{U}}_{i}\right\Vert _{2}^{2}}\right),
\end{aligned}
\end{equation*}
Setting $t=\lambda\eta$ leads to the result.

Now let $\bm{z}=\left(\tilde{\bm{U}}_{S}^{\top}\tilde{\bm{U}}_{S}\right)^{-1}\tilde{\bm{U}}_{S}^{\top}\tilde{\bm{U}}\vec{\bm{\varepsilon}}$, we have
\begin{equation*}
\begin{aligned} \mathrm{Var}\left(\bm{z}\right) &=\left(\tilde{\bm{U}}_{S}^{\top}\tilde{\bm{U}}_{S}\right)^{-1}\tilde{\bm{U}}_{S}^{\top}\tilde{\bm{U}}\mathrm{Var}\left(\vec{\bm{\varepsilon}}\right)\tilde{\bm{U}}^{\top}\tilde{\bm{U}}_{S}\left(\tilde{\bm{U}}_{S}^{\top}\tilde{\bm{U}}_{S}\right)^{-1}  \\ 
& \leq\sigma^{2}\left(\tilde{\bm{U}}_{S}^{\top}\tilde{\bm{U}}_{S}\right)^{-1}
 \leq\frac{\sigma^{2}}{C_{\min}}\bm{I}.
\end{aligned}
\end{equation*}
Then
\begin{equation*}
\mathbb{P}\left(\left\|\left(\tilde{\bU}_{S}^{\top} \tilde{\bU}_{S}\right)^{-1} \tilde{\bU}_{S}^{\top} \tilde{\bU}  \vec{\be}\right\|_{\infty} \geq t\right) \leq 2\left|S\right| \exp \left(-\frac{t^{2}  C_{\min } }{2 \sigma^{2}}\right).
\end{equation*}
Choose
\begin{equation}
t=\frac{\lambda \eta}{ \sqrt{C_{\min }} \max_{i\in S^{c}}\left\Vert \tilde{\bm{U}}_{i}\right\Vert _{2}},
\end{equation}
then there holds
\begin{equation*}
\begin{aligned}
& \mathbb{P}\left\{\|\left(\tilde{\bU}_{S}^{\top} \tilde{\bU}_{S}\right)^{-1} \tilde{\bU}_{S}^{\top} \tilde{\bU}  \vec{\be}\|_{\infty} \geq \frac{\lambda \eta}{ \sqrt{C_{\min }} \max_{i\in S^{c}}\left\Vert \tilde{\bm{U}}_{i}\right\Vert _{2}}\right\} \\
& \leq 2\left|S\right| \exp \left(-\frac{\lambda^{2} \eta^{2} }{2  \sigma^{2} \max_{i\in S^{c}}\left\Vert \tilde{\bm{U}}_{i}\right\Vert _{2}^{2}}\right).
\end{aligned}
\end{equation*}
\end{proof}
\begin{proof}[Proof of Theorem 1]
The proof essentially follows the treatment in~\cite{wainwright2009sharp} as well as the Huber's LASSO case in~\cite{xu2021evaluating}.
The results follow by applying Lemma~\ref{lemma} to Proposition~\ref{prop}.
Inequality~\eqref{appendix:prop-condition} holds if condition C2 and the first bound~\eqref{appendix:first-bound} hold, 
which proves the first part of the theorem.
The sign consistency~\eqref{appendix:sign-consisty} holds if condition C3 and the second bound~\eqref{appendix:second-bound} hold, 
which gives the second part of the theorem.

\minor{
It suﬃces to show that $\hat{S}\subseteq S$ implies $\hat{O}\subseteq O$.
Consider one instance $i$, there are three possible cases for $\bm{\gamma}_{i}^{*}\in\mathbb{R}^{1\times c}$: 
(1) $\gamma_{i,j}^{*}\neq0,\forall j\in\left[c\right]$; 
(2) $\gamma_{i,j}^{*}=0,\forall j\in\left[c\right]$; 
(3) $\exists j,k\in\left[c\right],s.t.\ \gamma_{i,j}^{*}=0,\gamma_{i,k}^{*}\neq0$.
If instance $i$ follows case (1) or case (3), then $i\in O$.
If it follows case (2), then $i\in O^{c}$,
and the indexes of all elements of $\bg_i$ are in $S^c$.
Since we have $\hat{S}\subseteq S$, all elements of $\bg_{i}$ is in $\hat{S}^{c}$, hence $i\in\hat{O}^{c}$. 
Then we have $\hat{O}\subseteq O$.
}
\end{proof}

\begin{IEEEbiography}[{\includegraphics[width=1in,height=1.25in,clip,keepaspectratio]{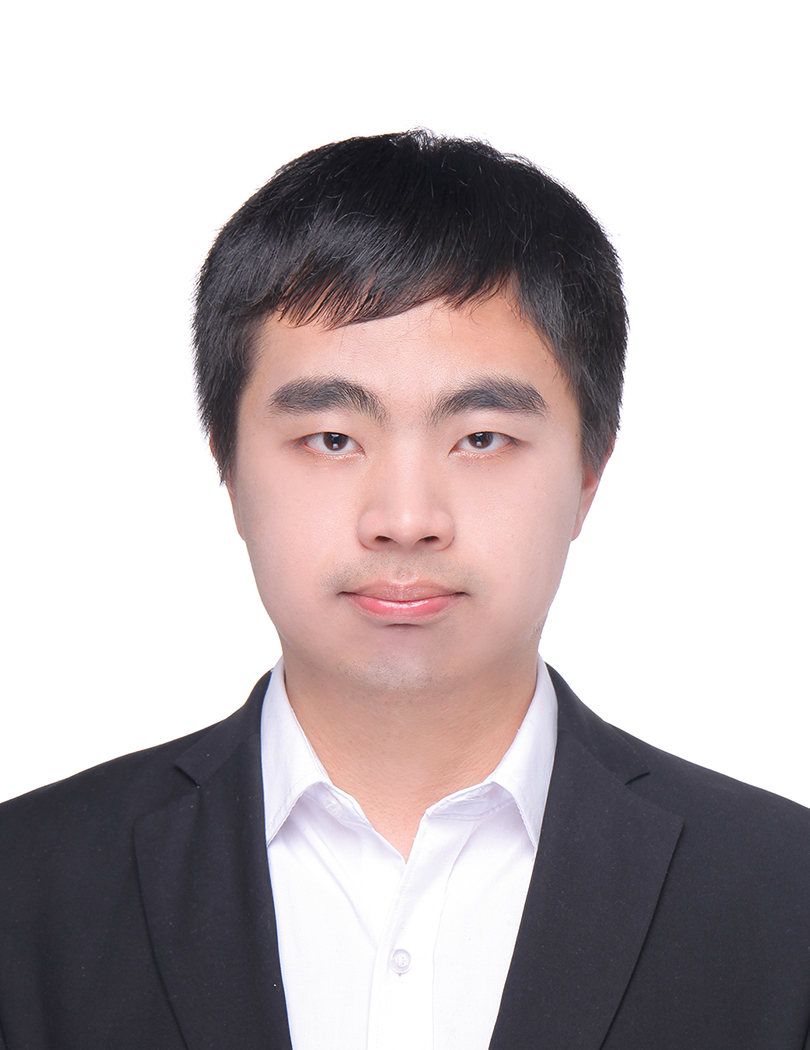}}]{Yikai Wang}
  is a PhD candidate at the School of Data Science, Fudan University.
  He works in Shanghai  Key  Lab  of  Intelligent  Information  Processing under the supervision of Prof. Yanwei Fu. 
  He received the Bachelor degree of mathematics from the School of Mathematical Sciences, Fudan University, in 2019.
  His current research interests include theoretically guaranteed machine learning algorithms and applications to computer vision.
  \end{IEEEbiography}
  
  \begin{IEEEbiography}[{\includegraphics[width=1in,height=1.25in,clip,keepaspectratio]{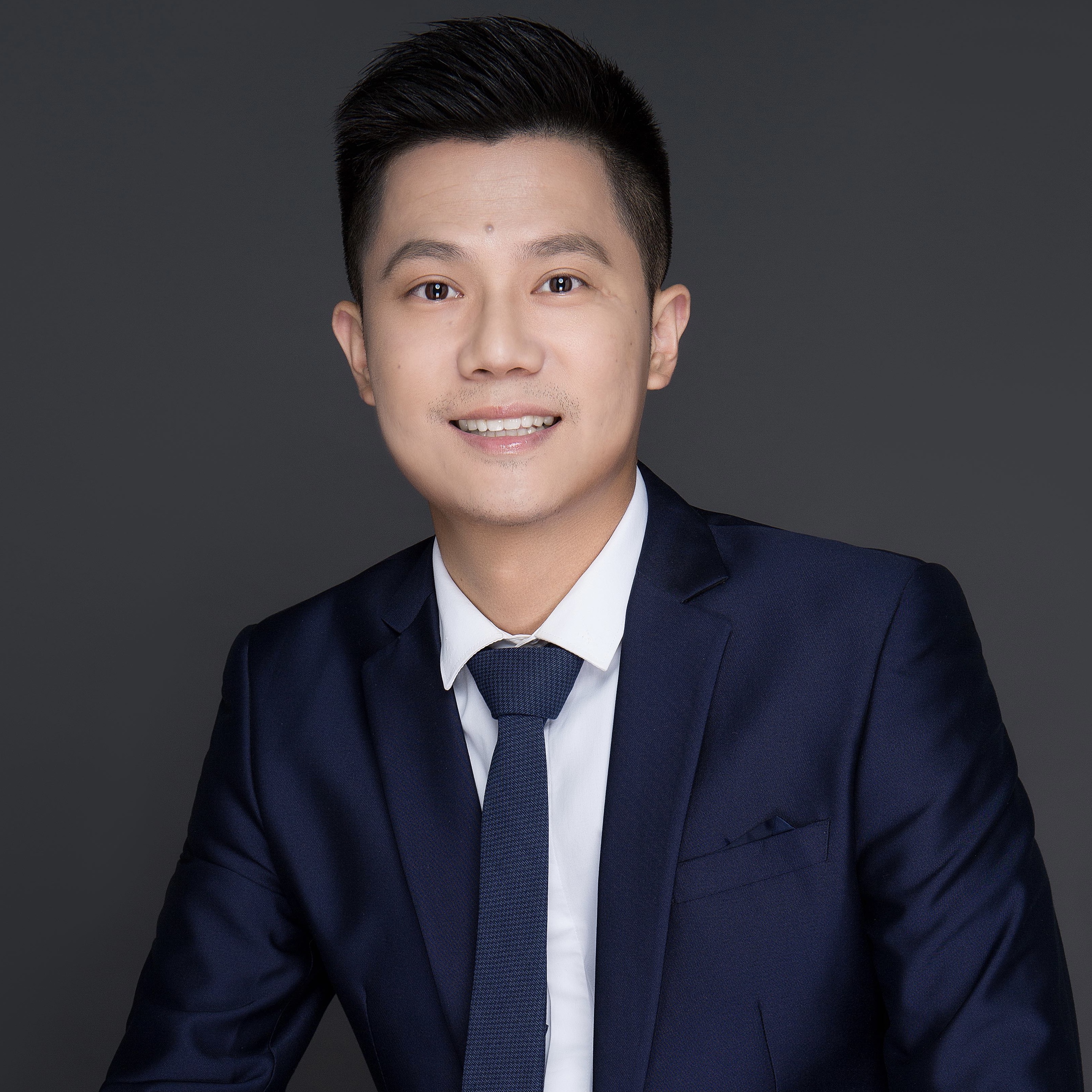}}]{Li Zhang} is a tenure-track Associate Professor at the School of Data Science, Fudan University. Previously, he was a Research Scientist at Samsung AI Center Cambridge, and a Postdoctoral Research Fellow at the University of Oxford. Prior to joining Oxford, he read his PhD in computer science at Queen Mary University of London.
  His research interests include computer vision and deep learning.
  \end{IEEEbiography}

  \begin{IEEEbiography}[{\includegraphics[width=1in,height=1.25in,clip,keepaspectratio]{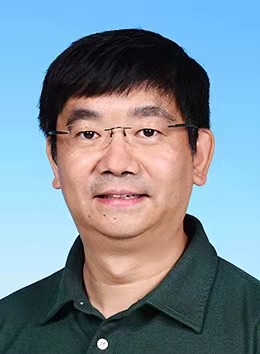}}]{Yuan Yao}
  received the B.S.E and M.S.E in control engineering both from Harbin Institute of Technology, China, in 1996 and 1998, respectively, M.Phil in mathematics from City University of Hong Kong in 2002, and Ph.D. in mathematics from the University of California, Berkeley, in 2006.
  Since then he has been with Stanford University and in 2009, he joined the Department of Probability and Statistics in School of Mathematical Sciences, Peking University, Beijing, China. 
  He is currently an Associate Professor of Mathematics, Chemical \& Biological Engineering, and by courtesy, Computer Science \& Engineering, Hong Kong University of Science and Technology, Clear Water Bay, Kowloon, Hong Kong SAR, China. 
  His current research interests include topological and geometric methods for high dimensional data analysis and statistical machine learning, with applications in computational biology, computer vision, and information retrieval. 
  Dr. Yao is a member of American Mathematical Society (AMS), Association for Computing Machinery (ACM), Institute of Mathematical Statistics (IMS), and Society for Industrial and Applied Mathematics (SIAM). 
  He served as area or session chair in NIPS and ICIAM, as well as a reviewer of Foundation of Computational Mathematics, IEEE Trans. Information Theory, J. Machine Learning Research, and Neural Computation, etc.
  \end{IEEEbiography}
  
  \begin{IEEEbiography}[{\includegraphics[width=1in,height=1.25in,clip,keepaspectratio]{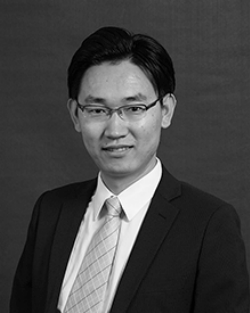}}]{Yanwei Fu} received the MEng degree from the Department of Computer Science and Technol- ogy, Nanjing University, China, in 2011, and the PhD degree from the Queen Mary University of London, in 2014. He held a post-doctoral position at Disney Research, Pittsburgh, PA, from 2015 to 2016. He is currently a tenure-track professor with Fudan University.  
  He was appointed as the Professor of Special Appointment (Eastern Scholar) at Shanghai Institutions of Higher Learning.
  His work has led to many awards, including the IEEE ICME 2019 best paper.
  He published more than 80 journal/conference papers including IEEE TPAMI, TMM, ECCV, and CVPR. His research interests are one-shot learning, and learning based 3D reconstruction.
  \end{IEEEbiography}
  
\end{document}